\documentclass{article}

\PassOptionsToPackage{numbers, compress}{natbib}
\usepackage[main,final]{neurips_2026}

\usepackage[utf8]{inputenc}
\usepackage[T1]{fontenc}
\usepackage{hyperref}
\usepackage{url}
\usepackage{booktabs}
\usepackage{amsfonts}
\usepackage{amsmath}
\usepackage{amssymb}
\usepackage{amsthm}
\usepackage{nicefrac}
\usepackage{microtype}
\usepackage[table]{xcolor}
\usepackage{graphicx}
\usepackage{subcaption}
\usepackage{natbib}
\usepackage{todonotes}
\usepackage{wrapfig}
\usepackage{tcolorbox}
\usepackage{enumitem}
\usepackage{multirow,multicol}
\usepackage{xspace}

\newtheorem{theorem}{Theorem}
\newtheorem{proposition}[theorem]{Proposition}

\usepackage{listings}
\usepackage[dvipsnames]{xcolor}

\newcommand{\name}{{\textsc{PORTAL}}\xspace}
\makeatletter
\newcommand{\shortmethod}[1][]{%
  \mathcal{R}_{\mathrm{S}}%
  \ifx\relax#1\relax
  \else
    ^{(#1)}%
  \fi
}
\makeatother

\definecolor{Gray}{gray}{0.90}
\newcolumntype{g}{>{\columncolor{Gray}}c}
\definecolor{ffe1da}{RGB}{255,225,218}
\definecolor{F7E0D5}{RGB}{247,224,213}
\definecolor{darkF7E0D5}{RGB}{209,154,128}
\colorlet{Light}{white!0!F7E0D5}

\lstdefinestyle{neuriPy}{
  language        = Python,
  basicstyle      = \ttfamily\footnotesize,
  keywordstyle    = \color{NavyBlue}\bfseries,
  commentstyle    = \color{OliveGreen}\itshape,
  stringstyle     = \color{BrickRed},
  showstringspaces= false,
  frame           = single,
  frameround      = tttt,
  rulecolor       = \color{black!35},
  backgroundcolor = \color{black!03},
  aboveskip       = 0.7\baselineskip,
  belowskip       = 0.7\baselineskip,
  tabsize         = 2,
  columns         = fullflexible,
  breaklines      = true,
  breakatwhitespace = true,
  captionpos      = b
}

\title{
    Rethinking Projector Training in Multimodal LLMs
}

\author{%
    Fabian Gr\"oger$^{1,2,3,}$\thanks{Joint first authorship}\; , Na\"el Ouerghemi$^{1,}$\footnotemark[1]\; , Shuo Wen$^{1,}$\footnotemark[1]\; , Maria Brbi\'c$^{1}$
    \\[0.2em]
    $^1$EPFL \; $^2$University of Basel \; $^3$HSLU
}

\begin{document}

\frenchspacing
\maketitle

\begin{abstract}
Multimodal LLMs are typically trained by first aligning a vision encoder with a language model via projector training on large image-caption datasets, followed by joint training of the full multimodal model.
Despite its widespread use, the role of the projector training stage remains poorly understood. 
We find that \textit{(i)} using only $10$ to $20\%$ of the data for training the projector already recovers most downstream performance gains, and
\textit{(ii)} projectors trained on different datasets can be linearly interpolated while largely preserving performance.
These results suggest that projector training mainly performs a coarse alignment, which many different projectors achieve equally well.
Motivated by these findings, we propose \name, a training-free projector initialization that requires \textit{neither paired image-caption data nor gradient-based optimization}, but only per-modality summary statistics.
\name computes an optimal-transport-based map between Gaussian approximations of the vision and language feature distributions on a shared principal subspace, so that the projector training stage can be skipped entirely.
Without any paired data, \name matches the standard pipeline trained on 558k image-caption pairs across six LLM backbones, two vision encoders, and 16 benchmarks. It outperforms that pipeline in $3$ of $7$ (vision encoder, LLM) settings and reaches an average $98\%$ of its score across the other four.
\end{abstract}

\section{Introduction}
Modern multimodal LLMs commonly align a vision encoder with a language model by first training a projector on a large amount of image-caption paired samples, followed by joint training of the full multimodal model.
This pipeline has become standard, from LLaVA~\cite{liu2023visual,liu2024llava}, Molmo~\cite{deitke2024molmo}, and InternVL~\cite{chen2024internvl} to the QwenVL series~\cite{bai2024qwen2vl}. 
While the first stage of projector alignment has been motivated by the goal of producing a visual tokenizer compatible with the frozen LLM~\cite{liu2023visual}, what this stage \textit{actually} does remains poorly understood. Prior works offer conflicting evidence: Prismatic~\cite{karamcheti2024prismatic} reports that skipping projector alignment can improve performance, whereas \mbox{Cambrian-1}~\cite{tong2024cambrian} shows that scaling the training data of projector alignment to 2.5M image-caption pairs helps. MM1~\cite{mckinzie2024mm1} suggests that factors like resolution and data mixing matter more than the projector design.
These findings are difficult to reconcile without a deeper understanding of what projector alignment actually does. 

In this work, we show that projector alignment primarily performs \textit{coarse alignment}, mapping visual features into a language-model-compatible input space which admits a large set of possible solutions.
Two empirical observations support this view:
\textit{(i)} using only $10$ to $20\%$ of the standard $558$k-sample alignment corpus achieves similar performance to training on the full dataset, and \textit{(ii)} projectors trained on entirely different datasets can be linearly interpolated without losing performance. 
If the projector alignment learned a precise and semantically unique mapping, neither reducing the training data nor linear interpolation between independently trained projectors would preserve performance. 

\begin{figure}
    \centering
    \includegraphics[width=1.0\linewidth]{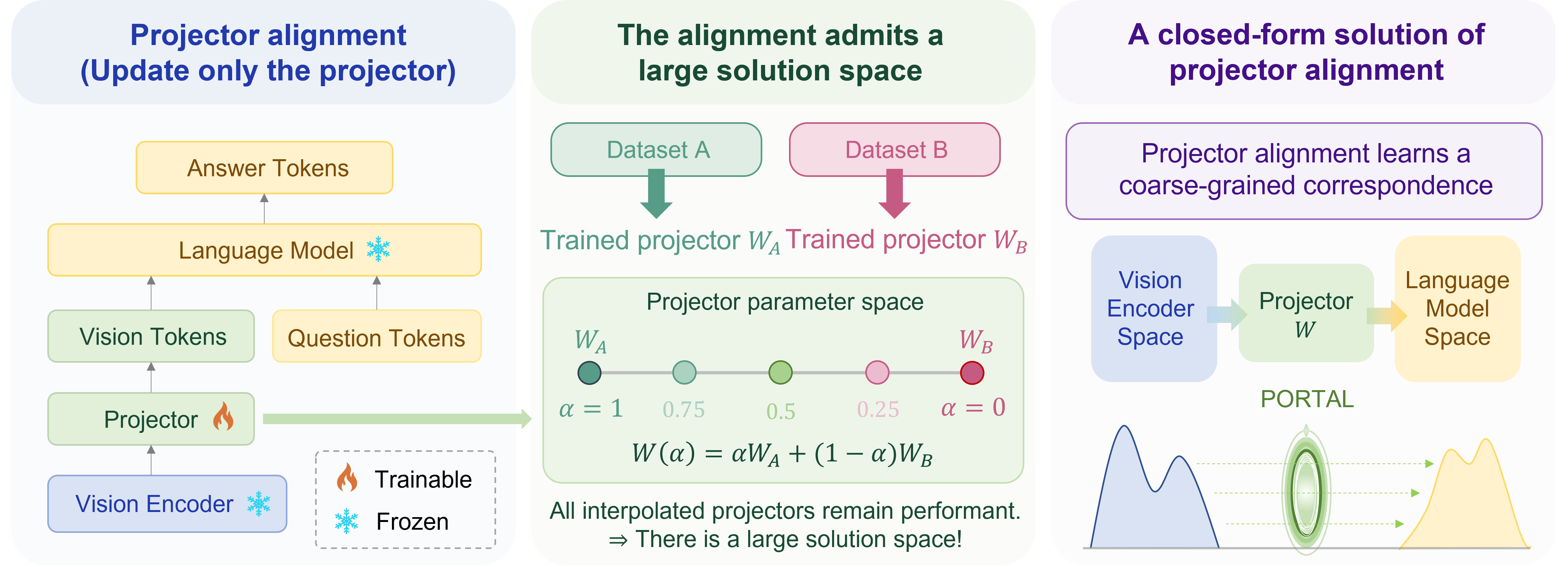}
    \caption{
    \textbf{Projector alignment as a coarse alignment with a broad solution space.}
    (\textit{Left}) Standard pipeline: during projector alignment, the vision encoder and language model are frozen, and only the projector is trained to map visual tokens into a representation space compatible with the LLM using image-caption pairs.
    (\textit{Middle}) We show that projectors trained on different image-caption datasets can be linearly interpolated while largely preserving downstream performance, indicating that there is a large solution space for the weights of the projector.
    (\textit{Right}) This observation motivates a closed-form solution of projector alignment. We propose a training-free projector \name constructed by optimal transport between the two modalities' marginal statistics.
    }
    \label{fig:main}
\end{figure}

Motivated by these observations, we propose PORTAL ({\textbf{P}rojector via \textbf{O}ptimal-transpo\textbf{RT} \textbf{AL}ignment}), a training-free approach that constructs the projector from the marginal statistics of the two modalities.
PORTAL is the Bures-Wasserstein optimal transport map between Gaussian approximations of the vision feature distribution and the language-model input distribution, computed in a shared top-$k$ principal subspace.
It does not require \textit{any paired captions or gradient-based optimization}, relying only on the per-modality means, principal directions, and the spread along each direction. PORTAL removes the projector alignment training stage entirely.

 Across seven (vision encoder, LLM) pairs spanning four LLM families (Vicuna, LLaMA, Qwen, OLMo) from 0.5B to 8B parameters, and two vision encoders (CLIP, SigLIP), PORTAL matches the standard pipeline trained on $558$k image-caption pairs and exceeds it in three out of seven settings, peaking at $106\%$ of the trained baseline. When PORTAL is used to initialize rather than replace projector alignment, the performance of the baseline is reached with $135\times$ fewer alignment FLOPs. Performance gaps are concentrated on OCR-heavy benchmarks, where marginal statistics appear insufficient to capture fine-grained character-level spatial structure that paired supervision provides.

\section{Related Work}\label{sec:RelatedWork}
\textbf{Projector alignment in multimodal LLMs.}
Two-stage multimodal LLM training, where a vision encoder is first aligned to a frozen language model via training a projector, followed by visual instruction tuning, has become a standard pipeline for training multimodal LLMs. 
This pipeline was popularized by LLaVA~\cite{liu2023visual,liu2024llava} and replicated in InstructBLIP \cite{dai2023instructblip}, BLIP-2 \cite{li2023blip2}, Flamingo \cite{alayrac2022flamingo}, Molmo \cite{deitke2024molmo}, InternVL \cite{chen2024internvl}, and the Qwen-VL series \cite{bai2024qwen2vl}. Despite its success, the role of Stage-I projector alignment has been poorly understood and prior studies offer conflicting views. Prismatic finds it can be skipped without a loss~\cite{karamcheti2024prismatic}. Cambrian-1 reports that scaling Stage-I data to 2.5M pairs helps~\cite{tong2024cambrian}. MM1 emphasizes resolution and data mixing over the projector itself~\cite{mckinzie2024mm1}. 
We resolve this disagreement by identifying what projector alignment actually does: Stage-I acts as \textit{geometric conditioning of visual features into the LLM input distribution}, not as fine-grained semantic learning. 
This explains the three findings above as a single phenomenon: skipping Stage-I works when the 
projector already lands in a compatible region (Prismatic~\cite{karamcheti2024prismatic}), more data helps only until that region is reached (Cambrian-1~\cite{tong2024cambrian}), and the resolution and data mixing dominate because they affect the conditioning more than the projector capacity does (MM1~\cite{mckinzie2024mm1}). 
These findings lead us to propose a closed-form projector that requires neither paired data nor gradient updates.

\textbf{Connector designs and projector geometry.}
Beyond the MLP used in LLaVA~\cite{liu2023visual}, connector architectures span the Q-Former \cite{li2023blip2}, Perceiver Resampler \cite{alayrac2022flamingo}, locality-enhanced abstractors \cite{cha2024honeybee}, ALIGN~\cite{masry2025alignvlm}, vision-expert mixtures \cite{zong2024mova}, prefix-attention adapters \cite{zhang2023llamaadapter}, and pixel-shuffle or patch-merging schemes for token reduction \cite{allal2025smolvlm,bai2024qwen2vl}. A complementary line of work analyzes what these connectors actually do. \citet{verma2024crossmodal} track how visual information degrades through LLM layers, \citet{li2025lostembeddings} probe information loss through learned projectors, \citet{li2025unseenbias} report systematic norm discrepancies between visual and text tokens, and \citet{shukor2024implicit} show that cross-modal alignment emerges primarily after self-attention layers. 
We use the 2-layer MLP projector of LLaVA by default, and also test PORTAL with a linear projector and with the C-Abstractor~\cite{cha2024honeybee}, which has $19\times$ more parameters (Section~\ref{sec:exp_robust}).

\textbf{Optimal transport for feature alignment.}
The Bures-Wasserstein optimal transport map between two Gaussians has a closed form on positive-definite covariances \cite{bhatia2019bures,peyre2019computational}, which the construction of the projector in our work exploits. \citet{courty2017optimal} introduced OT for domain adaptation, and \citet{alvarez2018gromov} used Gromov-Wasserstein to align word-embedding spaces with no paired anchors. 
Closest to \name is CORAL \cite{sun2016return,sun2016deep}, which aligns second-order statistics with a single linear map for domain adaptation. PORTAL extends this idea to the multimodal-alignment setting and ties it explicitly to the Bures-Wasserstein optimum on a shared top-$k$ principal subspace, which we then use as a closed-form projector rather than a regularizer.

\textbf{Training-free and efficient connector construction.}
Several recent works try to eliminate or reduce connector training. VPGTrans \cite{zhang2023vpgtrans} transfers a learned visual prompt generator across LLMs at reduced cost. LangBridge \cite{liao2025langbridge} interprets visual tokens as combinations of language embeddings, enabling cross-LLM transfer at the price of bottlenecking visual information through a vocabulary subspace. SEMI \cite{ince2025semi} meta-learns a hypernetwork connector that adapts to new modalities from a few samples. In the related CLIP-classification setting, training-free or lightweight adapters such as Tip-Adapter \cite{zhang2022tip} and TaskRes \cite{yu2023task} cache or residualize over a frozen backbone. PORTAL differs from all of the above by requiring neither paired data, gradient steps, learned caches, nor a vocabulary bottleneck: it is a closed-form construction from per-modality summary statistics alone. \citet{gupta2026better} also learn without paired multimodal data, but to improve a unimodal representation by sharing parameters across unpaired modalities, whereas we ask whether the projector between a frozen vision encoder and a frozen LLM needs paired training at all.

\section{What Does the Projector Alignment \textit{Actually} Do?}\label{sec:motivation}
\subsection{Projector Alignment Saturates with a Small Number of Samples}
\label{sec:task_level}

To investigate the role of projector alignment, we first examine how downstream performance scales with the amount of Stage-I training data. If strong performance requires large-scale data, it would suggest the projector is learning fine-grained visual-language correspondences. Conversely, if a small fraction of the data is sufficient, this would indicate that the projector primarily learns a coarse-grained correspondence, placing visual features into a broadly compatible input space for the language model.

We conduct this analysis on the LLaVA 1.5 setup (CLIP ViT-L/14-336, Vicuna 1.5 7B, 2-layer MLP projector), with LCS-558k for Stage-I and LLaVA-665k for Stage-II. For the data-fraction experiment we replace Vicuna with MobileLLaMA 1.4B, which keeps 11 data fractions on two corpora tractable. We evaluate on 16 benchmarks grouped into four categories adapted from the Cambrian-1~\cite{tong2024cambrian} taxonomy: Knowledge (Know.), General (Gen.), OCR, and Vision-Centric (Vis.). Appendix~\ref{app:eval_setup} lists the benchmarks in each. Full experimental details are in Appendix~\ref{app:motivation_setup}.

\begin{figure}
    \centering
    \begin{subfigure}{0.48\linewidth}
        \centering
        \includegraphics[width=\linewidth]{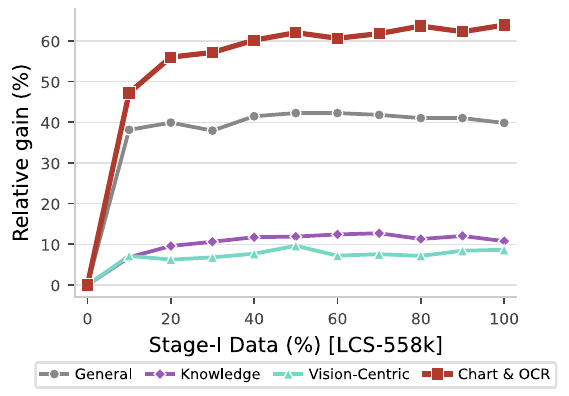}
        \caption{LCS-558k}
        \label{fig:task_category_lcs}
    \end{subfigure}
    \hfill
    \begin{subfigure}{0.48\linewidth}
        \centering
        \includegraphics[width=\linewidth]{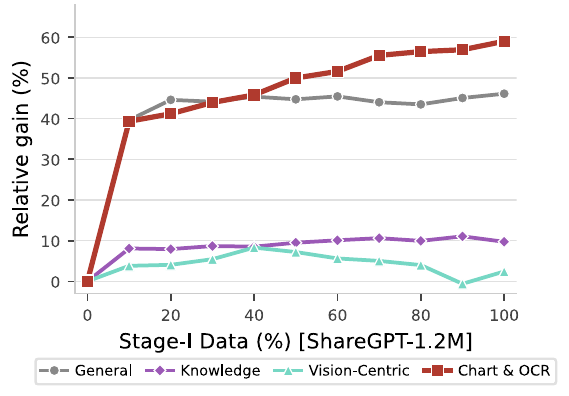}
        \caption{ShareGPT-1.2M}
        \label{fig:task_category_sharegpt}
    \end{subfigure}
    \caption{\textbf{Stage-I gain saturates early.} Relative gain in downstream performance over the no-pretraining baseline by Cambrian-1 category, as a function of Stage-I data fraction on LCS-558k (left) and ShareGPT-1.2M (right). OCR and Chart benefit the most, while Knowledge and Vision-Centric scores saturate early.}
    \label{fig:task_category}
\end{figure}

We observe that only $10$ to $20\%$ of the LCS-558k image-caption pairs already capture most of the downstream benefit, and all categories plateau by $30$ to $40\%$ (Figure \ref{fig:task_category}). The marginal downstream benefit of the remaining pairs is small under this protocol. ShareGPT-1.2M, more than twice as large, does not do better: its average peaks at $38.1$ with $40\%$ of its data (Appendix~\ref{app:full_benchmarks}).
These results suggest that projector alignment is not primarily learning fine-grained visual-language correspondences. Rather, it appears to perform a coarse-grained alignment step, mapping visual features into a region of the LLM input space from which subsequent joint training can effectively adapt the full multimodal model.

\subsection{Projector Alignment Admits a Broad Solution Space}
\label{sec:solution}

We have shown that projector alignment can be achieved with only a small fraction of the standard pretraining data. This observation is consistent with two possible explanations: either the set of effective projector solutions is narrow but easy to reach, or the solution space itself is broad. To distinguish between these possibilities, we examine whether independently trained projectors can be linearly interpolated while preserving performance.

\begin{wrapfigure}[20]{r}{0.42\linewidth}
\vspace{-1.2em}
\centering
\includegraphics[width=\linewidth]{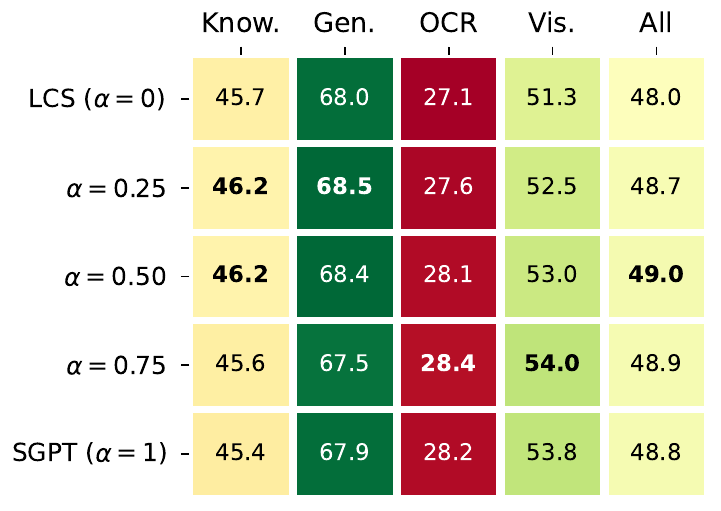}
\caption{\textbf{Cross-data projectors are linearly interpolable.} Cross-data interpolation $W_\alpha = (1{-}\alpha)W_{\text{LCS}} + \alpha W_{\text{ShareGPT}}$, evaluated per Cambrian-1 category. The optimal $\alpha$ varies by category, but \emph{every} interior blend outperforms both endpoints on the aggregate score, with the $\alpha{=}0.5$ mixture peaking above either trained projector.}
\label{fig:interp_heatmap}
\vspace{-1em}
\end{wrapfigure}

Specifically, we train two projectors on different datasets (LCS-558k versus ShareGPT-1.2M), obtaining the weights $W_{\text{LCS}}$ and $W_{\text{ShareGPT}}$. Then we evaluate the linear interpolation $W_\alpha = (1{-}\alpha)W_{\text{LCS}} + \alpha W_{\text{ShareGPT}}$ with $\alpha \in \{0, 0.25, 0.5, 0.75, 1\}$.
If effective projectors occupied a narrow and isolated region, interpolation would be expected to degrade performance. Instead, we find that projectors along the entire interpolation path largely preserve downstream performance, supporting the view that projector alignment admits a broad solution space (Figure~\ref{fig:interp_heatmap}).

The per-category breakdown results further support this interpretation.
The cross-data interpolation path is not only traversable but often outperforms either endpoint. Averaged across the four Cambrian-1 categories, every interior $\alpha$ improves on both LCS-558k and ShareGPT-1.2M projectors, with the optimum at $\alpha{=}0.5$.
The category-wise optima differ: OCR peaks closer to the ShareGPT endpoint, Knowledge closer to the LCS endpoint.
The aggregate score improves with mixing, so a simple combination of two independently trained solutions, with no additional Stage-I optimization, outperforms both.
Such behavior indicates that projector alignment admits a broad solution space.

The two projectors above start from the same initialization and differ only in their training data. We next asked whether the initialization matters in the same way, and trained a second pair on the same LCS-558k data from two different random initializations. Unlike the cross-data pair, this cross-seed pair is not linearly interpolable: captioning CIDEr drops from about $80$ at both endpoints to $0.7$ at $\alpha{=}0.5$ (Appendix~\ref{app:cross_seed}). Initialization thus appears to select the loss basin, and training data the solution within it.

These results clarify the role of projector alignment. The early saturation with limited data and the traversable interpolation path between projectors trained on different data suggest that projector alignment does not require a highly specific solution. Instead, projector alignment appears to place visual features into a broadly compatible region of the LLM input space, where many different projector weights can serve the same functional role.

\section{PORTAL: Training-Free Projector via Optimal Transport}
\label{sec:method_construction}

The previous findings suggest
that projector alignment primarily serves as geometric conditioning, placing visual features into a region of the LLM input space from which subsequent joint training can effectively adapt the multimodal model. This mapping admits a large set of solutions, motivating a simpler alternative to the projector alignment training stage. 

We instantiate this idea with \name (\textbf{P}rojector via \textbf{O}ptimal-transpo\textbf{RT} \textbf{AL}ignment), a training-free method that initializes the projector from the marginal feature statistics of the vision encoder and the LLM. \name estimates the mean, principal directions, and variances of each modality separately, then constructs a closed-form optimal-transport map that aligns the source visual feature distribution with a compatible target language-feature distribution. 
We also show why matching only means and covariances can suffice. The LLM normalizes each token before attention, so attention depends only on token directions. When token norms vary little, the spread of these directions is fixed by the covariance, which \name matches.

\subsection{Closed-Form Projector Initialization}
Let $\mathbf{x} \in \mathbb{R}^{d_s}$ denote a feature from the source modality (\textit{e.g.}, a per-patch vision token in our setting) and $\mathbf{y} \in \mathbb{R}^{d_t}$ a target embedding (\textit{e.g.}, an LLM input token). 
We seek a map $f: \mathbb{R}^{d_s} \to \mathbb{R}^{d_t}$, the \emph{projector} in standard multimodal-LLM terminology, whose output distribution is compatible with the distribution the target is trained to process.
We adopt the standard 2-layer MLP used by LLaVA~\citep{liu2024llava},
\begin{equation}
    f(\mathbf{x}) = \mathbf{W}_2\, \phi(\mathbf{W}_1 \mathbf{x} + \mathbf{b}_1) + \mathbf{b}_2,
    \quad \mathbf{W}_1 \in \mathbb{R}^{d_t \times d_s},\; \mathbf{W}_2 \in \mathbb{R}^{d_t \times d_t},
\end{equation}
with $\phi$ as GELU. We will fix $\mathbf{W}_2 = \mathbf{I}_{d_t}$ (\textit{i.e.}, the identity matrix of dimension $d_t$), $\mathbf{b}_2 = \mathbf{0}$ (\textit{i.e.}, the zero vector of dimension $d_t$), and construct $(\mathbf{W}_1, \mathbf{b}_1)$ in closed form.

\textbf{PORTAL.} Approximate each modality's marginal as a Gaussian, $\mathcal{N}(\boldsymbol\mu_s, \boldsymbol\Sigma_s)$ on the source side and $\mathcal{N}(\boldsymbol\mu_t, \boldsymbol\Sigma_t)$ on the target side. We derive the affine parameters $(\mathbf{W}_1, \mathbf{b}_1)$ from the Gaussian optimal transport map~\cite{olkin1982distance,bhatia2019bures} between these approximations, computed on the top-$k$ PCA subspaces of source and target with coordinates identified by eigenvalue order:
\begin{equation}
    \mathbf{W}_1 = \mathbf{V}_t \operatorname{diag}\left(\frac{\boldsymbol\sigma_t}{\boldsymbol\sigma_s}\right)\mathbf{V}_s^{\top},
    \qquad
    \mathbf{b}_1 = \boldsymbol\mu_t - \mathbf{W}_1 \boldsymbol\mu_s,
    \label{eq:wot}
\end{equation}
where $\mathbf{V}_s \in \mathbb{R}^{d_s \times k}$ and $\mathbf{V}_t \in \mathbb{R}^{d_t \times k}$ are the top-$k$ PCA bases of the source and target features (orthonormal columns) and $\boldsymbol\sigma_s, \boldsymbol\sigma_t \in \mathbb{R}^k$ are the square roots of the top-$k$ covariance eigenvalues, \textit{i.e.}, the spread along each principal axis. The map $\mathbf{W}_1$ whitens the source distribution (decorrelates it and rescales each axis to unit variance), expresses it in the target's principal axes, and rescales each axis to match the target's spread. The bias $\mathbf{b}_1$ recenters the result. The closed form matches Gaussian moments at the preactivation $\mathbf{W}_1\mathbf{x} + \mathbf{b}_1$. The post-GELU output is verified empirically in Appendix~\ref{app:gelu_diagnostic}. Appendix~\ref{app:portal_pseudocode} gives a self-contained PyTorch implementation. 

This construction depends only on summary statistics of each modality (\textit{i.e.}, means, principal directions, and the spread along each direction). It requires no Stage-I paired captions and no Stage-I optimization. A single PCA on each side is sufficient. It is agnostic to which model produces $\mathbf{x}$ and which model consumes $f(\mathbf{x})$. We adopt $k{=}512$ as the default throughout the paper, which is small enough to suppress noisy directions and large enough to preserve the dominant geometry. Varying $k$ over $\{64, 128, 256, 512, 1024\}$ supports this choice empirically (Appendix~\ref{app:rank_ablation}).

\textbf{On the Gaussian assumption.}
Eq.~\eqref{eq:wot} treats both modalities as Gaussian to obtain a closed-form transport map in terms of means, principal directions, and the spread along each direction. Real source and target distributions deviate from Gaussianity in tail weight and in the shape of their marginals (Appendix~\ref{app:gaussianity_audit}). We explain why the construction remains effective under these deviations in Section~\ref{sec:method_uniformity}.

\textbf{Targeting hidden states instead of input embeddings.}
Eq.~\eqref{eq:wot} matches visual features to the LLM's input embeddings, the vectors the LLM looks up for each text token before its first transformer block. The same construction can instead target the hidden states that text tokens have after block $\ell$: we replace $(\boldsymbol\mu_t, \mathbf{V}_t, \boldsymbol\sigma_t)$ by the mean, PCA basis, and spread along the principal axes of these hidden states, collected by passing a list of English words and short phrases through the LLM. The visual tokens still enter the LLM at its input. Only the target statistics change. Trained projectors that inject visual tokens into deeper layers~\cite{meng2024deepstack,qwenteam2025qwen3vl} motivate this variant, which we call PORTAL-ctx$\ell$. We use it to ablate the target layer (Section~\ref{sec:exp_ctx}), where the input embeddings ($\ell{=}0$) remain the best target. Details are in Appendix~\ref{app:ctx_methodology}.

\subsection{Why Moment Matching Survives Non-Gaussianity}
\label{sec:method_uniformity}

Real vision and text features are not Gaussian (Appendix~\ref{app:gaussianity_audit}), so Eq.~\eqref{eq:wot} cannot match their full distributions. We argue that it does not need to. Every LLM we use normalizes before computing attention logits (RMSNorm on each input token, or on queries and keys in OLMo~2), so rescaling a token leaves its logits unchanged. Attention therefore depends only on the directions of the tokens, and the relevant property of the projector output is how these directions are spread.

We measure this spread by the \emph{cosine dispersion} $\omega$, the standard deviation of the cosine similarity between two tokens:
\begin{equation}
  \omega^2 = \frac{1}{\binom{n}{2}} \sum_{i<j} \left( \cos\angle(\mathbf{z}_i, \mathbf{z}_j) - \overline{\cos} \right)^2,
  \qquad
  \overline{\cos} = \frac{1}{\binom{n}{2}} \sum_{i<j} \cos\angle(\mathbf{z}_i, \mathbf{z}_j),
  \label{eq:sigma_c}
\end{equation}
where $\{\mathbf{z}_i\}_{i=1}^n$ are the output tokens, pooled over all patches of all images. A small $\omega$ means the tokens point in similar directions, a large $\omega$ that they spread over many directions. Two properties of a distribution with covariance $\boldsymbol\Sigma$ determine $\omega$. The \emph{effective dimension} $d_{\mathrm{eff}}(\boldsymbol\Sigma) := (\mathrm{tr}\,\boldsymbol\Sigma)^2/\mathrm{tr}(\boldsymbol\Sigma^2)$ counts the directions that carry most of the variance: it equals $d$ when variance is spread evenly over $d$ directions and $1$ when a single direction holds it all. The \emph{norm spread} $\kappa := \mathrm{std}(\|\mathbf z\|^2)/\mathbb{E}[\|\mathbf z\|^2]$ is small when all tokens have about the same length.

\begin{proposition}[Spectrum-to-$\omega$ bound]
\label{prop:sigma_c}
Let $\mathbf z$ and $\mathbf z'$ be independent random vectors in $\mathbb{R}^d$ with the same distribution, which has mean zero and covariance $\boldsymbol\Sigma$. Assume $\mathbb{P}(\mathbf z = \mathbf 0) = 0$. Then
\begin{equation}
\left| \mathrm{Var}\big(\cos\angle(\mathbf z, \mathbf z')\big) - \frac{1}{d_{\mathrm{eff}}(\boldsymbol\Sigma)} \right| \le 2\kappa + \kappa^2 + \kappa^4 .
\label{eq:prop_bound}
\end{equation}
\end{proposition}

The assumption $\mathbb{P}(\mathbf z = \mathbf 0) = 0$ says that a token is never exactly the zero vector, so that its cosine with another token is always defined. The left side of Eq.~\eqref{eq:prop_bound} is the variance of the cosine similarity between two tokens drawn at random from the distribution, which is the population value of $\omega^2$ that Eq.~\eqref{eq:sigma_c} estimates from $n$ tokens. The right side is zero when all tokens have the same length ($\kappa = 0$), and the proposition then gives $\omega^2 = 1/d_{\mathrm{eff}}(\boldsymbol\Sigma)$: the spread of token directions is fixed by the covariance alone. Intuitively, when all tokens have the same length, the cosine of two tokens is their dot product divided by a constant, and the second moment of the dot product of two independent tokens depends only on the covariance. The proof is in Appendix~\ref{app:concentration} and does not assume Gaussianity. \name matches the target's mean and covariance on the top-$k$ subspace (second-order matching, \textit{i.e.}, matching the means and covariances of the two distributions), so in this limit its outputs inherit the target's $\omega$.

Real token norms vary (Appendix~\ref{app:gelu_diagnostic}), so the bound does not predict $\omega$ precisely for our features. We therefore measure $\omega$ directly. On three (CLIP, LLM) pairs, \name reduces $\omega$ from $0.23$ at the source to $0.09$ to $0.10$, towards the target's $0.05$ to $0.06$, while a random projector leaves $\omega$ at the source value (Appendix~\ref{app:sigma_c_projector}). The attention weights also depend on the learned query and key projections, which $\omega$ does not describe.

\section{Experiments}\label{sec:Experiments}

\definecolor{bestcell}{HTML}{E8F5E9}
\definecolor{avgrow}{HTML}{F5F5F5}
\definecolor{baserow}{HTML}{EEEEEE}
\definecolor{altrow}{HTML}{FAFAFA}
\newcolumntype{G}{!{\color{gray!45}\vrule width 0.5pt}}

\begin{table*}[t]
\centering
\caption{\textbf{PORTAL generalizes across LLM backbones and vision encoders.}
All configurations use identical Stage-II instruction tuning (LLaVA-665k~\cite{liu2024llava}).
PORTAL requires zero Stage-I paired captions and zero Stage-I gradient steps.
Baseline $\dag$: uses 558k paired captions for Stage-I pretraining.
\textbf{Bold}: best among zero training data methods.
\textbf{Avg}: unweighted mean of the 16 per-benchmark scores. \%BL: this average as a percentage of the trained baseline (BL), which is therefore $100\%$.
Subscripts on the Avg and \%BL columns give the 95\% bootstrap confidence-interval half-width over per-sample outcomes, from 1000 resamples paired across methods.
LLM backbones: Vicuna 1.5~7B~\cite{chiang2023vicuna}, LLaMA~3.1~8B~\cite{dubey2024llama3}, Qwen~2.5~7B~\cite{yang2024qwen2}, LLaMA~3.2~1B~\cite{dubey2024llama3}, Qwen~2.5~0.5B~\cite{yang2024qwen2}, OLMo~2~1B~\cite{olmo2_2025}.
Vision encoders: CLIP~ViT-L/14-336~\cite{radford2021clip}, SigLIP~SO400M~\cite{zhai2023siglip}.
$^\text{L}$: 500-sample Lite subset from \texttt{lmms-eval}~\cite{zhang2024lmmseval}.}
\label{tab:main_results}
\vspace{2pt}
\resizebox{\textwidth}{!}{%
\setlength{\tabcolsep}{3.5pt}
\renewcommand{\arraystretch}{1.08}
\begin{tabular}{@{}l c l l G ccccG cccc G cccc G cccc G c  c@{}}
\toprule
& & & & \multicolumn{4}{c}{Knowledge} & \multicolumn{4}{c}{General} & \multicolumn{4}{c}{OCR} & \multicolumn{4}{c}{Vision} & & \\
\cmidrule(lr){5-8} \cmidrule(lr){9-12} \cmidrule(lr){13-16} \cmidrule(lr){17-20}
\textbf{LLM} & \textbf{Scale} & \textbf{Vision} & \textbf{Init} & \rotatebox{60}{AI2D} & \rotatebox{60}{SciQA} & \rotatebox{60}{MathV} & \rotatebox{60}{MMMU} & \rotatebox{60}{GQA} & \rotatebox{60}{MME} & \rotatebox{60}{SEED$^\text{L}$} & \rotatebox{60}{MMB} & \rotatebox{60}{TxtVQA} & \rotatebox{60}{OCR-Bv2} & \rotatebox{60}{DocVQA} & \rotatebox{60}{ChartQA} & \rotatebox{60}{CV-2D} & \rotatebox{60}{CV-3D} & \rotatebox{60}{BLINK} & \rotatebox{60}{RealWQ} & \cellcolor{avgrow}\textbf{Avg} & \textbf{\%BL} \\
\midrule
\multirow{6}{*}{Vicuna 1.5} & \multirow{6}{*}{7B}
 & \multirow{3}{*}{CLIP}
 & \cellcolor{baserow}Baseline$^\dag$ & \cellcolor{baserow}54.5 & \cellcolor{baserow}67.9 & \cellcolor{baserow}26.1 & \cellcolor{baserow}34.3 & \cellcolor{baserow}66.6 & \cellcolor{baserow}74.5 & \cellcolor{baserow}67.7 & \cellcolor{baserow}63.2 & \cellcolor{baserow}57.3 & \cellcolor{baserow}12.5 & \cellcolor{baserow}22.0 & \cellcolor{baserow}16.5 & \cellcolor{baserow}54.7 & \cellcolor{baserow}58.3 & \cellcolor{baserow}37.4 & \cellcolor{baserow}54.9 & \cellcolor{baserow}48.0 & \cellcolor{baserow}100.0 \\
 & & & Random        & 51.5 & \cellcolor{bestcell}\textbf{69.1} & 27.2 & 32.8 & 64.2 & 69.2 & 57.2 & 58.1 & 52.8 & 11.0 & 17.1 & 15.2 & 52.5 & 53.8 & 34.7 & 51.9 & \cellcolor{avgrow}44.9$_{\pm 0.60}$ & 93.5$_{\pm 1.6}$ \\
 & & & PORTAL      & \cellcolor{bestcell}\textbf{53.4} & 68.2 & \cellcolor{bestcell}\textbf{27.4} & \cellcolor{bestcell}\textbf{34.4} & \cellcolor{bestcell}\textbf{66.4} & \cellcolor{bestcell}\textbf{69.9} & \cellcolor{bestcell}\textbf{67.9} & \cellcolor{bestcell}\textbf{60.4} & \cellcolor{bestcell}\textbf{55.1} & \cellcolor{bestcell}\textbf{12.7} & \cellcolor{bestcell}\textbf{20.5} & \cellcolor{bestcell}\textbf{16.3} & \cellcolor{bestcell}\textbf{56.2} & \cellcolor{bestcell}\textbf{60.7} & \cellcolor{bestcell}\textbf{36.0} & \cellcolor{bestcell}\textbf{53.5} & \cellcolor{avgrow}\cellcolor{bestcell}\textbf{47.4}$_{\pm 0.56}$ & \cellcolor{bestcell}\textbf{98.8}$_{\pm 1.5}$ \\
\cmidrule(lr){3-22}
 & & \multirow{3}{*}{SigLIP}
 & \cellcolor{baserow}Baseline$^\dag$ & \cellcolor{baserow}55.8 & \cellcolor{baserow}68.0 & \cellcolor{baserow}27.2 & \cellcolor{baserow}34.0 & \cellcolor{baserow}68.7 & \cellcolor{baserow}73.6 & \cellcolor{baserow}69.1 & \cellcolor{baserow}66.0 & \cellcolor{baserow}61.5 & \cellcolor{baserow}13.3 & \cellcolor{baserow}25.3 & \cellcolor{baserow}19.1 & \cellcolor{baserow}57.5 & \cellcolor{baserow}60.8 & \cellcolor{baserow}28.1 & \cellcolor{baserow}56.6 & \cellcolor{baserow}49.0 & \cellcolor{baserow}100.0 \\
 & & & Random        & 51.8 & 67.9 & 26.7 & \cellcolor{bestcell}\textbf{35.1} & 56.8 & 60.6 & 46.5 & 52.3 & 43.8 & 7.9 & 8.9 & 12.8 & 47.1 & 57.5 & 25.6 & 45.9 & \cellcolor{avgrow}40.5$_{\pm 0.58}$ & 82.5$_{\pm 1.1}$ \\
 & & & PORTAL      & \cellcolor{bestcell}\textbf{55.5} & \cellcolor{bestcell}\textbf{68.3} & \cellcolor{bestcell}\textbf{27.9} & 34.4 & \cellcolor{bestcell}\textbf{67.8} & \cellcolor{bestcell}\textbf{71.6} & \cellcolor{bestcell}\textbf{68.3} & \cellcolor{bestcell}\textbf{64.3} & \cellcolor{bestcell}\textbf{60.0} & \cellcolor{bestcell}\textbf{13.6} & \cellcolor{bestcell}\textbf{24.8} & \cellcolor{bestcell}\textbf{17.9} & \cellcolor{bestcell}\textbf{57.6} & \cellcolor{bestcell}\textbf{64.6} & \cellcolor{bestcell}\textbf{28.6} & \cellcolor{bestcell}\textbf{54.2} & \cellcolor{avgrow}\cellcolor{bestcell}\textbf{48.7}$_{\pm 0.56}$ & \cellcolor{bestcell}\textbf{99.3}$_{\pm 1.5}$ \\
\midrule
\multirow{3}{*}{LLaMA 3.1} & \multirow{3}{*}{8B} & \multirow{3}{*}{CLIP}
 & \cellcolor{baserow}Baseline$^\dag$ & \cellcolor{baserow}58.5 & \cellcolor{baserow}73.8 & \cellcolor{baserow}31.3 & \cellcolor{baserow}36.8 & \cellcolor{baserow}67.0 & \cellcolor{baserow}73.0 & \cellcolor{baserow}71.7 & \cellcolor{baserow}68.1 & \cellcolor{baserow}57.2 & \cellcolor{baserow}13.7 & \cellcolor{baserow}21.8 & \cellcolor{baserow}15.1 & \cellcolor{baserow}64.5 & \cellcolor{baserow}63.0 & \cellcolor{baserow}39.8 & \cellcolor{baserow}57.5 & \cellcolor{baserow}50.8 & \cellcolor{baserow}100.0 \\
 & & & Random        & 58.6 & 74.8 & \cellcolor{bestcell}\textbf{31.7} & 35.3 & 67.6 & 73.4 & 69.5 & 67.5 & 56.2 & 13.9 & 19.6 & 15.0 & 61.8 & 66.2 & 38.5 & \cellcolor{bestcell}\textbf{57.0} & \cellcolor{avgrow}50.4$_{\pm 0.58}$ & 99.2$_{\pm 1.4}$ \\
 & & & PORTAL      & \cellcolor{bestcell}\textbf{62.2} & \cellcolor{bestcell}\textbf{76.2} & 30.7 & \cellcolor{bestcell}\textbf{38.6} & \cellcolor{bestcell}\textbf{67.7} & \cellcolor{bestcell}\textbf{74.1} & \cellcolor{bestcell}\textbf{69.9} & \cellcolor{bestcell}\textbf{70.1} & \cellcolor{bestcell}\textbf{58.0} & \cellcolor{bestcell}\textbf{14.5} & \cellcolor{bestcell}\textbf{22.7} & \cellcolor{bestcell}\textbf{16.8} & \cellcolor{bestcell}\textbf{65.2} & \cellcolor{bestcell}\textbf{70.9} & \cellcolor{bestcell}\textbf{40.4} & 56.9 & \cellcolor{avgrow}\cellcolor{bestcell}\textbf{52.2}$_{\pm 0.57}$ & \cellcolor{bestcell}\textbf{102.7}$_{\pm 1.5}$ \\
\midrule
\multirow{3}{*}{LLaMA 3.2} & \multirow{3}{*}{1B} & \multirow{3}{*}{CLIP}
 & \cellcolor{baserow}Baseline$^\dag$ & \cellcolor{baserow}46.5 & \cellcolor{baserow}66.2 & \cellcolor{baserow}26.9 & \cellcolor{baserow}33.6 & \cellcolor{baserow}59.7 & \cellcolor{baserow}58.4 & \cellcolor{baserow}56.2 & \cellcolor{baserow}46.1 & \cellcolor{baserow}43.7 & \cellcolor{baserow}10.1 & \cellcolor{baserow}12.8 & \cellcolor{baserow}10.8 & \cellcolor{baserow}53.5 & \cellcolor{baserow}56.4 & \cellcolor{baserow}37.2 & \cellcolor{baserow}47.3 & \cellcolor{baserow}41.6 & \cellcolor{baserow}100.0 \\
 & & & Random        & 46.6 & 66.9 & 26.3 & \cellcolor{bestcell}\textbf{34.2} & 59.8 & 58.1 & 56.0 & 45.2 & 43.8 & 10.4 & 12.8 & 11.0 & 54.2 & 54.9 & 37.4 & 46.9 & \cellcolor{avgrow}41.5$_{\pm 0.57}$ & 99.9$_{\pm 0.7}$ \\
 & & & PORTAL      & \cellcolor{bestcell}\textbf{47.6} & \cellcolor{bestcell}\textbf{67.1} & \cellcolor{bestcell}\textbf{29.2} & 33.7 & \cellcolor{bestcell}\textbf{62.1} & \cellcolor{bestcell}\textbf{66.0} & \cellcolor{bestcell}\textbf{59.6} & \cellcolor{bestcell}\textbf{54.2} & \cellcolor{bestcell}\textbf{47.1} & \cellcolor{bestcell}\textbf{11.4} & \cellcolor{bestcell}\textbf{16.3} & \cellcolor{bestcell}\textbf{12.0} & \cellcolor{bestcell}\textbf{54.6} & \cellcolor{bestcell}\textbf{57.3} & \cellcolor{bestcell}\textbf{38.9} & \cellcolor{bestcell}\textbf{49.7} & \cellcolor{avgrow}\cellcolor{bestcell}\textbf{44.2}$_{\pm 0.60}$ & \cellcolor{bestcell}\textbf{106.2}$_{\pm 1.3}$ \\
\midrule
\multirow{3}{*}{Qwen 2.5} & \multirow{3}{*}{7B} & \multirow{3}{*}{CLIP}
 & \cellcolor{baserow}Baseline$^\dag$ & \cellcolor{baserow}66.9 & \cellcolor{baserow}77.5 & \cellcolor{baserow}33.3 & \cellcolor{baserow}44.7 & \cellcolor{baserow}65.6 & \cellcolor{baserow}74.6 & \cellcolor{baserow}70.7 & \cellcolor{baserow}71.2 & \cellcolor{baserow}57.2 & \cellcolor{baserow}15.6 & \cellcolor{baserow}23.0 & \cellcolor{baserow}19.0 & \cellcolor{baserow}54.0 & \cellcolor{baserow}69.0 & \cellcolor{baserow}42.3 & \cellcolor{baserow}56.6 & \cellcolor{baserow}52.6 & \cellcolor{baserow}100.0 \\
 & & & Random        & 62.1 & 73.9 & 27.7 & 42.0 & 63.7 & 70.9 & 67.7 & 60.1 & 51.0 & 13.6 & 17.6 & 14.6 & 56.3 & 60.0 & \cellcolor{bestcell}\textbf{42.1} & 53.9 & \cellcolor{avgrow}48.6$_{\pm 0.56}$ & 92.4$_{\pm 0.9}$ \\
 & & & PORTAL      & \cellcolor{bestcell}\textbf{65.7} & \cellcolor{bestcell}\textbf{76.7} & \cellcolor{bestcell}\textbf{31.9} & \cellcolor{bestcell}\textbf{44.8} & \cellcolor{bestcell}\textbf{65.5} & \cellcolor{bestcell}\textbf{75.9} & \cellcolor{bestcell}\textbf{70.9} & \cellcolor{bestcell}\textbf{68.0} & \cellcolor{bestcell}\textbf{55.4} & \cellcolor{bestcell}\textbf{14.9} & \cellcolor{bestcell}\textbf{18.9} & \cellcolor{bestcell}\textbf{17.1} & \cellcolor{bestcell}\textbf{57.9} & \cellcolor{bestcell}\textbf{64.8} & 41.5 & \cellcolor{bestcell}\textbf{56.1} & \cellcolor{avgrow}\cellcolor{bestcell}\textbf{51.7}$_{\pm 0.58}$ & \cellcolor{bestcell}\textbf{98.3}$_{\pm 1.4}$ \\
\midrule
\multirow{3}{*}{Qwen 2.5} & \multirow{3}{*}{0.5B} & \multirow{3}{*}{CLIP}
 & \cellcolor{baserow}Baseline$^\dag$ & \cellcolor{baserow}43.6 & \cellcolor{baserow}58.6 & \cellcolor{baserow}26.7 & \cellcolor{baserow}31.7 & \cellcolor{baserow}52.2 & \cellcolor{baserow}53.9 & \cellcolor{baserow}47.9 & \cellcolor{baserow}38.7 & \cellcolor{baserow}35.7 & \cellcolor{baserow}11.1 & \cellcolor{baserow}9.4 & \cellcolor{baserow}9.7 & \cellcolor{baserow}45.7 & \cellcolor{baserow}50.7 & \cellcolor{baserow}39.4 & \cellcolor{baserow}46.0 & \cellcolor{baserow}37.6 & \cellcolor{baserow}100.0 \\
 & & & Random        & 43.5 & 58.5 & 26.8 & \cellcolor{bestcell}\textbf{33.1} & 52.5 & 54.2 & 48.9 & 37.9 & 35.6 & 10.5 & 9.5 & 9.8 & \cellcolor{bestcell}\textbf{44.2} & \cellcolor{bestcell}\textbf{51.5} & \cellcolor{bestcell}\textbf{40.5} & 45.5 & \cellcolor{avgrow}37.7$_{\pm 0.58}$ & 100.2$_{\pm 1.0}$ \\
 & & & PORTAL      & \cellcolor{bestcell}\textbf{44.2} & \cellcolor{bestcell}\textbf{59.1} & \cellcolor{bestcell}\textbf{27.4} & 31.3 & \cellcolor{bestcell}\textbf{58.1} & \cellcolor{bestcell}\textbf{57.6} & \cellcolor{bestcell}\textbf{56.0} & \cellcolor{bestcell}\textbf{46.4} & \cellcolor{bestcell}\textbf{41.8} & \cellcolor{bestcell}\textbf{11.8} & \cellcolor{bestcell}\textbf{14.6} & \cellcolor{bestcell}\textbf{11.6} & 41.4 & 51.2 & 39.7 & \cellcolor{bestcell}\textbf{45.6} & \cellcolor{avgrow}\cellcolor{bestcell}\textbf{39.9}$_{\pm 0.56}$ & \cellcolor{bestcell}\textbf{106.1}$_{\pm 1.3}$ \\
\midrule
\multirow{3}{*}{OLMo 2} & \multirow{3}{*}{1B} & \multirow{3}{*}{CLIP}
 & \cellcolor{baserow}Baseline$^\dag$ & \cellcolor{baserow}46.9 & \cellcolor{baserow}64.3 & \cellcolor{baserow}25.0 & \cellcolor{baserow}33.9 & \cellcolor{baserow}62.7 & \cellcolor{baserow}63.1 & \cellcolor{baserow}63.2 & \cellcolor{baserow}57.6 & \cellcolor{baserow}46.3 & \cellcolor{baserow}10.5 & \cellcolor{baserow}17.2 & \cellcolor{baserow}12.2 & \cellcolor{baserow}57.3 & \cellcolor{baserow}54.8 & \cellcolor{baserow}37.6 & \cellcolor{baserow}48.1 & \cellcolor{baserow}43.8 & \cellcolor{baserow}100.0 \\
 & & & Random       & 45.7 & 58.8 & 24.6 & 31.2 & 59.4 & 59.5 & 59.8 & 49.4 & 42.4 & 9.7 & 13.6 & 11.0 & 50.2 & 50.9 & \cellcolor{bestcell}\textbf{37.5} & 43.9 & \cellcolor{avgrow}40.5$_{\pm 0.57}$ & 92.8$_{\pm 1.0}$ \\
 & & & PORTAL      & \cellcolor{bestcell}\textbf{47.2} & \cellcolor{bestcell}\textbf{61.8} & \cellcolor{bestcell}\textbf{26.8} & 31.2 & \cellcolor{bestcell}\textbf{61.8} & \cellcolor{bestcell}\textbf{61.8} & \cellcolor{bestcell}\textbf{60.2} & \cellcolor{bestcell}\textbf{51.5} & \cellcolor{bestcell}\textbf{44.2} & \cellcolor{bestcell}\textbf{9.8} & \cellcolor{bestcell}\textbf{15.0} & \cellcolor{bestcell}\textbf{11.1} & \cellcolor{bestcell}\textbf{54.8} & \cellcolor{bestcell}\textbf{55.3} & 36.2 & \cellcolor{bestcell}\textbf{46.3} & \cellcolor{avgrow}\cellcolor{bestcell}\textbf{42.2$_{\pm 0.59}$} & \cellcolor{bestcell}\textbf{96.3$_{\pm 1.0}$} \\
\bottomrule
\end{tabular}}
\end{table*}

We test whether \name, built without any paired image-caption data, matches projector alignment trained on 558k image-caption pairs across LLM backbones and vision encoders. 

We evaluate this question by inserting a PORTAL-constructed projector into the standard LLaVA pipeline and then running the same LLaVA-665k instruction tuning for every configuration. We compare against two baselines: \textit{(i)} the standard trained LLaVA projector, obtained from 558k paired captions, and \textit{(ii)} a randomly initialized projector which skips projector alignment entirely. This setup isolates the contribution of projector initialization: after initialization, all models receive identical joint training.

We test seven (vision encoder, LLM) pairs, built from three 7B to 8B LLMs (Vicuna~1.5~7B~\cite{chiang2023vicuna}, LLaMA~3.1~8B~\cite{dubey2024llama3}, Qwen~2.5~7B~\cite{yang2024qwen2}), three smaller LLMs for cross-scale checks (LLaMA~3.2~1B, Qwen~2.5~0.5B, OLMo~2~1B~\cite{olmo2_2025}), and two vision encoders (CLIP ViT-L/14-336~\cite{radford2021clip}, SigLIP-SO400M~\cite{zhai2023siglip}). Table~\ref{tab:main_results} lists the pairs. 
The trained baseline uses the standard LLaVA Stage-I corpus of 558k LCS image-caption pairs, and all configurations share the same LLaVA-665k instruction-tuning data for Stage-II. The 16 evaluation benchmarks are grouped into the four Cambrian-1 categories. The full benchmark list, citations, and evaluation protocol are in Appendix~\ref{app:eval_setup}.

\subsection{PORTAL Matches Trained Projector Alignment Without Paired Data}
\label{sec:exp_main}

We find that our closed-form projector \name requiring \textit{no paired captions or gradient updates} achieves comparable downstream performance to conventional alignment pipelines trained on 558k image-caption pairs, across a range of vision encoder and LLM combinations (Table~\ref{tab:main_results}).
In particular, using LLaMA~3.2~1B/CLIP and Qwen~2.5~0.5B/CLIP, PORTAL reaches $\mathbf{106.2\%}$ and $\mathbf{106.1\%}$ of the trained baseline, respectively. Using LLaMA~3.1~8B/CLIP, it exceeds the trained baseline, reaching $\mathbf{102.7\%}$. On Vicuna~1.5 with CLIP and SigLIP, Qwen~2.5~7B, and OLMo~2~1B, PORTAL remains very close to the trained projector, reaching $98.8\%$, $99.3\%$, $98.3\%$, and $96.3\%$ of the trained baseline, respectively.
Thus, PORTAL removes the requirement for paired training data and optimization cost of projector alignment while preserving almost all of its downstream benefit after instruction tuning.

The advantage of PORTAL is clearest where the random projector fails. On SigLIP, the Random baseline collapses to $\mathbf{82.5\%}$ of the trained projector. PORTAL closes that gap to $\mathbf{99.3\%}$. This is what we would expect if the alignment stage were dominated by distributional matching rather than caption supervision.

Two settings sit below the trained baseline's bootstrap CI, Qwen 2.5 7B/CLIP and OLMo 2 1B/CLIP, with the gaps concentrated on OCR and Chart benchmarks. This matches the observation in Section~\ref{sec:task_level} that OCR and Chart drive most of the Stage-I gain over the no-pretraining baseline, supporting our reading that character-level spatial structure is the regime where paired captions teach what marginal statistics alone cannot recover. 
Importantly, this phenomenon is not specific to PORTAL.
Other unpaired distribution-matching approaches, including CORAL~\cite{sun2016deep} and
Gromov-Wasserstein~\cite{alvarez2018gromov} alignment, also recover most of the performance of paired Stage-I training (Appendix~\ref{app:alignment_baselines}).
This provides additional evidence that much of the benefit of Stage-I can be obtained by bringing the two modalities into a compatible feature distribution, rather than requiring a uniquely learned image-text correspondence.

A category-level analysis across all seven configurations further shows that this behavior is not systematic across OCR tasks: PORTAL retains, on average, 100.7\%, 102.0\%, 102.4\%, and 100.4\% of the trained-baseline performance on Knowledge, General, OCR, and Vision-Centric tasks, respectively (Appendix~\ref{app:category_breakdown}).
We further test PORTAL beyond the standard LLaVA-665k Stage-II setting.
Under substantially larger-scale multimodal instruction tuning with 3.2M samples, PORTAL slightly outperforms the conventionally Stage-I-trained initialization (Appendix~\ref{app:large_sft}).
We also evaluate object hallucination on POPE~\cite{li2023evaluating} and find that PORTAL matches or exceeds the trained Stage-I baseline in every setting (Appendix~\ref{app:pope}).

\subsection{PORTAL Reduces the Projector Alignment Budget}
\label{sec:exp_data}

\begin{wrapfigure}[21]{r}{0.45\linewidth}
    \vspace{-2.2em}
    \centering
    \includegraphics[width=\linewidth]{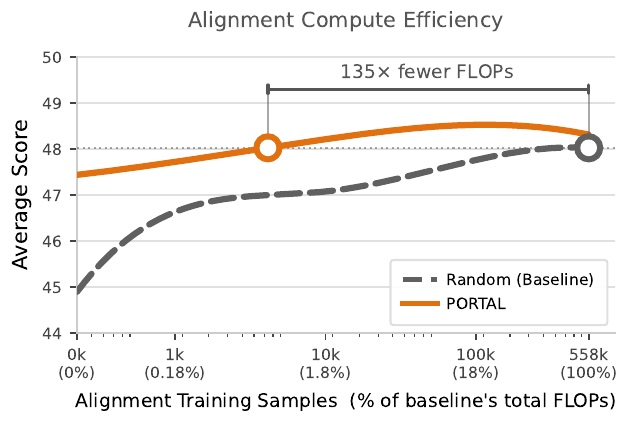}
    \caption{\textbf{Alignment compute efficiency} on Vicuna~1.5~7B/CLIP. 
    Average over the 16 benchmarks as a function of the projector alignment budget, with identical Stage-II tuning on LLaVA-665k. PORTAL stays above Random at every budget. 
    The marked point indicates the crossover at roughly $4$k captions, where PORTAL matches the $558$k-trained baseline at $\sim$$135{\times}$ fewer alignment FLOPs. 
    The horizontal dashed line shows the $558$k baseline.
    }
    \label{fig:data_efficiency}
    \vspace{-1.0em}
\end{wrapfigure}

The previous experiment shows that PORTAL can provide a strong projector without paired optimization. We next ask whether PORTAL can also serve as a warm start for projector alignment. Instead of using PORTAL to skip this stage entirely, we initialize the projector with PORTAL and then continue projector alignment with progressively smaller subsets of LCS-558k. This experiment evaluates whether closed-form alignment can reduce the amount of paired data needed for subsequent projector training.

We vary the projector alignment budget over $\{0, 1\text{k}, 10\text{k}, 100\text{k}, 558\text{k}\}$ paired captions from LCS-558k on Vicuna 1.5 7B/CLIP, starting from either a Random or a PORTAL projector. The $0$ point skips projector alignment entirely and proceeds directly to Stage-II from the initialization, and the $558$k point is the standard LLaVA baseline. Stage-II tuning on LLaVA-665k is identical for every configuration.

PORTAL provides a better starting point at every projector alignment budget than the random initialization (Figure~\ref{fig:data_efficiency}). 
With zero projector training, it already outperforms random initialization. As more paired captions are added, the PORTAL-initialized projector remains consistently above the randomly initialized projector. PORTAL reaches the performance of the full 558k trained projector baseline using only around 4k captions, corresponding to roughly $\mathbf{135{\times}}$ fewer projector alignment FLOPs.
Beyond matching the trained baseline at lower compute, PORTAL with as few as 10k paired captions exceeds the full 558k-trained projector baseline and continues to widen the gap up to 100k.
PORTAL therefore offers two practical advantages: it reaches the trained baseline at substantially lower alignment cost, and it surpasses that baseline once even a small paired-data budget is available.

This result strengthens the main claim that much of projector alignment is a distributional conditioning problem. The projector has a large solution space, and many mappings can make visual features compatible with the LLM input space. PORTAL identifies one such compatible solution directly in closed form. This solution is not necessarily the exact optimum selected by supervised projector training but it lies close to the high-performing region, providing strong performance without paired optimization and a small amount of paired data can further refine it within this region. Additional projector training thus mainly refines a projector that is already compatible.

\newpage
\subsection{Ablations: Projector Architecture, Image Pool, and Rank}
\label{sec:exp_robust}

\begin{wraptable}[13]{r}{0.42\linewidth}
\vspace{-1.0em}
\centering
\setlength{\tabcolsep}{4pt}
\renewcommand{\arraystretch}{1.05}
\caption{\textbf{Linear vs.\ MLP-2x-GELU projector} on LLaMA~3.2~1B/CLIP. Full per-benchmark breakdown in Appendix~\ref{app:linear_vs_mlp}.}
\label{tab:linear_vs_mlp_inline}
\begin{tabular}{l l c c}
\toprule
\textbf{Projector} & \textbf{Init} & \textbf{Avg} & \textbf{\%BL} \\
\midrule
\multirow{3}{*}{MLP} & \cellcolor{baserow}Trained$^\dag$ & \cellcolor{baserow}$41.6$ & \cellcolor{baserow}$100.0$ \\
                              & Random         & \cellcolor{avgrow}\underline{$41.5$} & \underline{$99.9$} \\
                              & PORTAL   & \cellcolor{avgrow}\cellcolor{bestcell}$\mathbf{44.2}$ & \cellcolor{bestcell}$\mathbf{106.2}$ \\
\midrule
\multirow{3}{*}{Linear} & \cellcolor{baserow}Trained$^\dag$ & \cellcolor{baserow}$39.9$ & \cellcolor{baserow}$100.0$ \\
                                 & Random         & \cellcolor{avgrow}\underline{$37.6$} & \underline{$94.2$} \\
                                 & PORTAL   & \cellcolor{avgrow}\cellcolor{bestcell}$\mathbf{44.7}$ & \cellcolor{bestcell}$\mathbf{111.8}$ \\
\bottomrule
\end{tabular}
\vspace{-0.8em}
\end{wraptable}

A key advantage of PORTAL is that it is a closed-form geometric construction rather than a learned projector tied to a particular architecture or paired dataset. We therefore test whether its gains persist when the projector is simplified and when the unlabeled image pool used to estimate vision-side statistics is changed.

For architecture, we replace the standard 2-layer MLP-GELU projector with a single linear layer, keeping all other components fixed (Table~\ref{tab:linear_vs_mlp_inline}). If PORTAL's gains depended on the expressivity of the MLP, they should weaken in this setting. Instead, the gain becomes larger. With the MLP projector, PORTAL reaches $\mathbf{106.2\%}$ of the trained baseline against Random's $99.9\%$. With the linear projector, it reaches $\mathbf{111.8\%}$ against Random's $94.2\%$, widening the Random-to-PORTAL gap from $6.3$ to $17.6$ percentage points. This shows that PORTAL's advantage does not come from a particular nonlinearity or hidden expansion. It comes from placing the visual features into a better-matched LLM input distribution.
We also test the opposite regime by substantially increasing connector capacity.
Using the C-Abstractor~\cite{cha2024honeybee}, which contains approximately $19\times$ more parameters than the two-layer MLP and additionally performs token mixing and sequence compression, PORTAL remains essentially indistinguishable from conventional Stage-I training (42.93 vs. 42.98 average score; Appendix~\ref{app:cabstractor}).

For the image pool, we vary the unlabeled image pool used to estimate the vision-side statistics. On LLaMA~3.2~1B/CLIP, we compare a small single-domain pool, COCO \texttt{val2017} with $1024$ images, against a larger and more diverse Cambrian-737k subset spanning $9$ sources with $4096$ images. The small COCO pool matches or exceeds the diversified pool on $12$ of $15$ benchmarks and obtains a higher average score ($\mathbf{44.0}$ vs.\ $43.4$). This suggests that PORTAL does not require a carefully curated or highly diverse image pool. A small unlabeled natural-image pool is sufficient to estimate the source statistics well enough for downstream instruction tuning in this setting. Per-benchmark results are reported in Appendix~\ref{app:image_pool_ablation}.

Additionally, the PORTAL construction depends on a rank $k$ used to define the PCA subspace. We vary $k$ over $\{64, 128, 256, 512, 1024\}$ on LLaMA~3.2~1B/CLIP, constructing each projector from the same $1024$ COCO images and running identical Stage-II tuning. The method is robust to this choice (Appendix~\ref{app:rank_ablation}): every $k$ outperforms the trained baseline by at least $\sim$$4\%$. Performance is flat from $k{=}64$ to $k{=}256$ ($\sim$$43.3$ Avg), peaks at $k{=}512$ ($44.2$), and decreases marginally at $k{=}1024$ ($44.0$). The plateau at low ranks indicates that the dominant PCA directions already capture most of the cross-modal structure. The small drop at $k{=}1024$ is consistent with noisy eigenvectors diluting the transport map. We adopt $k{=}512$ as the default.

\subsection{Target Layer Choice}
\label{sec:exp_ctx}

\begin{wraptable}[12]{r}{0.40\linewidth}
\vspace{-2.2em}
\centering
\setlength{\tabcolsep}{4pt}
\renewcommand{\arraystretch}{1.05}
\caption{\textbf{Layer ablation for PORTAL-ctx} on LLaMA~3.1~8B/CLIP. Avg over the 16-benchmark suite, \%BL relative to trained Stage-I.}
\label{tab:ctx_layer_inline}
\begin{tabular}{l c c}
\toprule
\textbf{Target} & \cellcolor{avgrow}\textbf{Avg} & \textbf{\%BL} \\
\midrule
\cellcolor{baserow}Trained$^\dag$ & \cellcolor{baserow}$51.3$ & \cellcolor{baserow}$100.0$ \\
Random                 & \cellcolor{avgrow}$51.4$ & $100.2$ \\
\midrule
$\ell{=}0$ (embed)     & \cellcolor{avgrow}\cellcolor{bestcell}$\mathbf{52.5}$ & \cellcolor{bestcell}$\mathbf{102.3}$ \\
$\ell{=}4$             & \cellcolor{avgrow}\underline{$52.3$} & \underline{$101.9$} \\
$\ell{=}8$             & \cellcolor{avgrow}$52.0$ & $101.4$ \\
\bottomrule
\end{tabular}
\vspace{-0.8em}
\end{wraptable}

PORTAL targets the LLM's input embedding distribution by default. The PORTAL-ctx$\ell$ variant (Section~\ref{sec:method_construction}) generalizes the construction to deeper targets, matching visual features to hidden states after LLM block $\ell$ instead. We use this variant to ablate the choice of target layer.

Targeting the embedding table ($\ell=0$) is the strongest choice on LLaMA~3.1~8B/CLIP, reaching $\mathbf{102.3\%}$ of the trained projector alignment baseline (Table~\ref{tab:ctx_layer_inline}). Contextualized targets at $\ell=4$ and $\ell=8$ remain close, reaching $101.9\%$ and $101.4\%$, respectively. We therefore use the embedding table as the default target.

\section{Conclusion}\label{sec:Conclusion}
This work rethinks projector alignment in multimodal LLMs. Our results suggest that this stage primarily performs coarse alignment, mapping visual features into a language-model-compatible input space rather than learning fine-grained visual-language correspondences. This view is supported by the early saturation of alignment data, where 10 to 20\% of LCS-558k recovers most downstream gains, and by the linear interpolability of projectors trained on different datasets from a shared initialization.

Motivated by this finding, we introduced PORTAL, a training-free projector initialization. PORTAL uses only per-modality summary statistics, requires no paired image-caption data or gradient-based optimization, and remains close to the 558k-trained baseline across the evaluated model pairs. As an initialization, it reaches the trained baseline with roughly $135\times$ lower alignment compute. These results suggest that projector alignment can often be replaced by, or reduced to, a lightweight feature-distribution matching step, lowering the cost and paired-data requirements for building multimodal LLMs.

Several natural directions remain for future work. Our study focuses on projector alignment with frozen vision and language backbones during alignment, single-image inputs, and simple projector architectures. Extending the analysis to richer connectors, long-context grounding, multi-image reasoning, and additional modalities such as audio, video, and 3D would further clarify how broadly this coarse-alignment view applies. We discuss other limitations of \name in Appendix \ref{sec:Limitations}.

\section*{Acknowledgements}
We thank Yulun Jiang, Siba Smarak Panigrahi, and Petar Damjanović for their valuable suggestions, which helped improve the clarity of the manuscript.
We gratefully acknowledge the support of the Swiss National Science Foundation (SNSF) starting grant TMSGI2\_226252/1, SNSF grant IC00I0\_231922, SNSF grant 10.004.411, Swiss Cancer League (KFS-6340-02-2025-R), Kamprad Foundation (Cancer Immunotherapy Networking Program), and the Swiss AI Initiative Large Call (ab006). Maria Brbi\'c is a CIFAR Fellow in the Multiscale Human Program.

\medskip
\small
\bibliography{bibliography}

\appendix
\newpage
\section{Limitations}\label{sec:Limitations}
Our analysis and method address a specific slice of the multimodal-LLM pipeline: Stage-I projector alignment with frozen vision and language backbones. Several caveats follow from that scope.
\textit{(i)} The empirical analysis in Section~\ref{sec:motivation} runs on the LLaVA~1.5 architecture with a 2-layer MLP projector. Richer connectors (resamplers, perceiver-style cross-attention, multi-stage abstractors) may not exhibit the same linear interpolability, and we do not test them here.
\textit{(ii)} PORTAL is constructed from a Gaussian approximation of each modality. The approximation captures means, principal directions, and per-axis spread, but no higher-order structure. Appendix~\ref{app:gaussianity_audit} quantifies how far real source and target distributions deviate from this assumption. We hypothesize that the residual gap on OCR and Chart benchmarks reflects this missing higher-order structure (character-level spatial layout).
\textit{(iii)} As in LLaVA-1.5, all experiments keep the vision encoder frozen and train only the projector and the LLM in Stage-II. We do not study end-to-end fine-tuning that also updates the vision encoder.
\textit{(iv)} Our 16-benchmark suite follows the Cambrian-1 grouping and inherits its biases. It does not stress long-context grounding or multi-image reasoning, which may behave differently.
\textit{(v)} We instantiate PORTAL for vision-to-language alignment. We have not tested it on audio or video encoders. The present results should be regarded as a first existence proof rather than a finished construction.
\textit{(vi)} Our reading of Stage-I as geometric conditioning is the simplest hypothesis consistent with our data. We do not claim it is the unique mechanism, and other interpretations remain possible.
\textit{(vii)} The closed-form construction exceeds trained Stage-I or stays within $1.6$ average points of it across the configurations we test. We make no broader claim that this holds across all multimodal-LLM regimes. Richer benchmarks or backbones may reveal regimes where trained Stage-I genuinely outperforms PORTAL.

\section{Training Details and Compute Resources}
\label{app:compute}

All experiments follow the standard two-stage LLaVA-1.5~\cite{liu2024llava} training recipe. Training is performed on $8\times$ NVIDIA A100 80\,GB GPUs with DeepSpeed ZeRO-2 and bf16 mixed precision.

\paragraph{Hyperparameters.}
Table~\ref{tab:hyperparams} summarizes the training configuration for both stages. All runs use the AdamW optimizer with $(\beta_1, \beta_2) = (0.9, 0.999)$ and no gradient clipping.

\begin{table}[!htbp]
\centering
\caption{\textbf{Training hyperparameters.} Identical to LLaVA-1.5.}
\label{tab:hyperparams}
\small
\begin{tabular}{@{}lcc@{}}
\toprule
& \textbf{Stage I (Pretrain)} & \textbf{Stage II (Finetune)} \\
\midrule
Learning rate         & $1 \times 10^{-3}$ & $2 \times 10^{-5}$ \\
Schedule              & Cosine             & Cosine              \\
Warmup                & 200 steps          & 3\% ratio           \\
Weight decay          & $5 \times 10^{-2}$ & 0                   \\
Per-device batch size  & 32                 & 4                   \\
Gradient accumulation  & 1                  & 4                   \\
Effective batch size   & 256                & 128                 \\
Epochs                & 1                  & 1                   \\
Trainable parameters  & Projector only     & Projector + LLM     \\
\bottomrule
\end{tabular}
\end{table}

\paragraph{Compute budget.}
Each Stage-I run completes in under one hour for MobileLLaMA~1.4B and in under four hours for Vicuna~7B on $8\times$A100 GPUs. Stage-II instruction tuning takes between two and ten hours depending on model scale ($2.4$\,h for Qwen~0.5B, $5.4$\,h for LLaMA~8B, $10$\,h for Qwen~7B). The cost of constructing a PORTAL projector is negligible: the closed-form solution requires only forward passes on approximately $1{,}024$ images and completes in under $30$ seconds on a single RTX~3090.

\paragraph{Seeds.}
The cross-seed analysis in Section~\ref{sec:solution} compares two Stage-I projectors trained with seeds $42$ and $120$ on the same LCS-558k data. All other experiments use seed $42$. Bootstrap confidence intervals are computed from $B{=}1{,}000$ resamples paired across methods within each (encoder, LLM) group.

\paragraph{Datasets and licenses.}
Table~\ref{tab:datasets} lists the datasets used in our experiments.

\begin{table}[!htbp]
\centering
\caption{\textbf{Datasets used in our experiments.}}
\label{tab:datasets}
\small
\setlength{\tabcolsep}{3pt}
\begin{tabular}{@{}llll@{}}
\toprule
\textbf{Dataset} & \textbf{Used for} & \textbf{License} & \textbf{Source} \\
\midrule
\multicolumn{4}{@{}l}{\textit{Training data}} \\
\midrule
LCS-558k          & Stage-I pretrain    & CC-BY-4.0       & \cite{liu2024llava} \\
ShareGPT-1.2M     & Stage-I pretrain    & CC-BY-NC-4.0    & \cite{chen2024sharegpt4v} \\
LLaVA-665k        & Stage-II finetune   & CC-BY-4.0       & \cite{liu2024llava} \\
Cambrian-737k     & PORTAL image pool      & Apache-2.0      & \cite{tong2024cambrian} \\
COCO val2017       & PORTAL image pool      & CC-BY-4.0       & \cite{lin2014coco} \\
\midrule
\multicolumn{4}{@{}l}{\textit{Knowledge benchmarks}} \\
\midrule
AI2D               & Evaluation          & CC-BY-SA-4.0    & \cite{kembhavi2016ai2d} \\
ScienceQA          & Evaluation          & CC-BY-NC-SA-4.0 & \cite{lu2022learn} \\
MathVista          & Evaluation          & CC-BY-SA-4.0    & \cite{lu2024mathvista} \\
MMMU               & Evaluation          & Apache-2.0      & \cite{yue2024mmmu} \\
\midrule
\multicolumn{4}{@{}l}{\textit{General benchmarks}} \\
\midrule
GQA                & Evaluation          & CC-BY-4.0       & \cite{hudson2019gqa} \\
MME                & Evaluation          & Research only   & \cite{fu2024mme} \\
SeedBench          & Evaluation          & Apache-2.0      & \cite{li2024seed} \\
MMBench            & Evaluation          & CC-BY-4.0       & \cite{liu2025mmbench} \\
\midrule
\multicolumn{4}{@{}l}{\textit{OCR \& Chart benchmarks}} \\
\midrule
TextVQA            & Evaluation          & CC-BY-4.0       & \cite{singh2019towards} \\
OCRBench-v2        & Evaluation          & CC-BY-NC-4.0    & \cite{liu2024ocrbenchv2} \\
DocVQA             & Evaluation          & Apache-2.0      & \cite{mathew2021docvqa} \\
ChartQA            & Evaluation          & GPL-3.0         & \cite{masry2022chartqa} \\
\midrule
\multicolumn{4}{@{}l}{\textit{Vision-Centric benchmarks}} \\
\midrule
CV-Bench           & Evaluation          & Apache-2.0      & \cite{tong2024cambrian} \\
BLINK              & Evaluation          & Apache-2.0      & \cite{fu2024blink} \\
RealWorldQA        & Evaluation          & CC-BY-ND-4.0    & \cite{xai2024realworldqa} \\
\bottomrule
\end{tabular}
\end{table}

\section{Evaluation Setup: Benchmarks and Protocol}
\label{app:eval_setup}

All models are evaluated on $16$ benchmarks grouped into four categories adapted from the Cambrian-1~\cite{tong2024cambrian} taxonomy (Knowledge, General, OCR, Vision), summarized in Table~\ref{tab:bench_suite}. Relative to the original Cambrian-1 suite, we replace OCRBench with OCRBench-v2~\cite{liu2024ocrbenchv2} and MMVP with BLINK~\cite{fu2024blink}. All evaluations are performed using \texttt{lmms-eval} v0.4.1~\cite{zhang2024lmmseval}.

\begin{table}[!htbp]
\centering
\caption{\textbf{Evaluation benchmarks grouped by Cambrian-1 category.} $^\text{L}$\,=\,Lite subset (500 samples).}
\label{tab:bench_suite}
\small
\begin{tabular}{@{}llrl@{}}
\toprule
\textbf{Category} & \textbf{Benchmark} & \textbf{Samples} & \textbf{Metric} \\
\midrule
\multirow{4}{*}{Knowledge}
  & AI2D~\cite{kembhavi2016ai2d}        & 3{,}088  & Exact match \\
  & ScienceQA~\cite{lu2022learn}        & 2{,}017  & Exact match \\
  & MathVista~\cite{lu2024mathvista}     & 1{,}000  & LLM-as-judge \\
  & MMMU~\cite{yue2024mmmu}             &    900   & Accuracy \\
\midrule
\multirow{4}{*}{General}
  & GQA~\cite{hudson2019gqa}             & 12{,}578 & Exact match \\
  & MME~\cite{fu2024mme}                 & 2{,}374  & Perception / 20 \\
  & SeedBench$^\text{L}$~\cite{li2024seed} &    500 & Accuracy (image) \\
  & MMBench~\cite{liu2025mmbench}        & 4{,}329  & GPT-eval score \\
\midrule
\multirow{4}{*}{OCR}
  & TextVQA~\cite{singh2019towards}      & 5{,}000  & Exact match \\
  & OCRBench-v2~\cite{liu2024ocrbenchv2} & 1{,}000  & Accuracy \\
  & DocVQA~\cite{mathew2021docvqa}       & 5{,}349  & ANLS \\
  & ChartQA~\cite{masry2022chartqa}      & 2{,}500  & Relaxed accuracy \\
\midrule
\multirow{4}{*}{Vision}
  & CV-Bench-2D~\cite{tong2024cambrian}  & 1{,}438  & Accuracy \\
  & CV-Bench-3D~\cite{tong2024cambrian}  & 1{,}200  & Accuracy \\
  & BLINK~\cite{fu2024blink}             & 1{,}437  & Accuracy \\
  & RealWorldQA~\cite{xai2024realworldqa}&    765   & Exact match \\
\bottomrule
\end{tabular}
\end{table}

\paragraph{Scoring conventions.}
MME perception scores (out of $2{,}000$) are divided by $20$ to fit a $0$--$100$ scale. MMBench reports GPT-eval scores ($0$--$100$) directly. MathVista uses GPT-4o as the judge model. All other metrics are reported as percentages ($0$--$100$).

\paragraph{Lite subsets.}
SeedBench uses the $500$-sample Lite split from \texttt{lmms-eval}, denoted $^\text{L}$ in Table~\ref{tab:main_results}. Running the full SeedBench ($17{,}990$ samples) for every model configuration is computationally prohibitive.

\paragraph{Category averages.}
The Avg column in Table~\ref{tab:main_results} is the arithmetic mean of all $16$ benchmarks. Since each of the four categories contains four benchmarks, this is equivalent to the macro-average of the four category means.

\paragraph{Bootstrap confidence intervals.}
$95\%$ confidence intervals for all category averages and overall scores are computed using $B{=}1{,}000$ bootstrap resamples~\cite{efron1979bootstrap}. Resampling is paired across methods within each (encoder, LLM) group: the same bootstrap indices are applied to all models that share a backbone, so per-sample correlations are preserved and the resulting intervals reflect method differences rather than sample noise. The \%BL column reports the ratio of each model's overall score to its group baseline, bootstrapped under the same paired protocol.

\subsection{Setup for the Data-Fraction and Interpolation Experiments}
\label{app:motivation_setup}

Throughout Section~\ref{sec:motivation}, we use the LLaVA~1.5 architecture. The vision encoder is CLIP ViT-L/14 (336px), the projector is a 2-layer MLP ($1024 \to 4096 \to 4096$), and the LLM is Vicuna~1.5~7B. The data-fraction experiment in Section~\ref{sec:task_level} uses the smaller MobileLLaMA~1.4B backbone~\cite{chu2024mobilevlm} with a $1024 \to 2048 \to 2048$ projector to keep the $11$ data fractions on two corpora computationally feasible. Per-benchmark numbers for this experiment are reported in Appendix~\ref{app:lcs558k}.
By default, we use LCS-558k (LAION/CC/SBU filtered pairs with BLIP captions) for Stage-I training and the standard LLaVA-665k instruction-tuning mix for Stage-II.
Section~\ref{sec:Experiments} extends the picture to LLaMA~3.1~8B and Qwen~2.5~7B, and to SigLIP-SO400M as the vision encoder.
For two corpora, LCS-558k and ShareGPT-1.2M, we train Stage-I projectors at evenly spaced data fractions ($0$ to $100\%$), apply identical LLaVA-665k instruction tuning on each, and evaluate on $16$ benchmarks grouped into the four Cambrian-1 categories (Knowledge, General, OCR, Vision-Centric).

\subsection{Cross-Seed vs.\ Cross-Data Interpolation}
\label{app:cross_seed}

Section~\ref{sec:solution} reports that projectors trained on different datasets from the same initialization remain interpolable along the linear interpolation path. This appendix complements that result with a control that varies the random initialization instead of the training data, and contrasts the two cases.

\paragraph{Setup.}
We train two pairs of projectors using the LLaVA 1.5 setup described in Appendix~\ref{app:motivation_setup}. \emph{Cross-data} projectors share the same random initialization but are trained on different corpora (LCS-558k and ShareGPT-1.2M). \emph{Cross-seed} projectors share the LCS-558k corpus but use two different random initializations. For each pair $(W_A, W_B)$ we report the Frobenius distance $\|W_A - W_B\|_F$ and evaluate the linear interpolation $W_\alpha = (1-\alpha)\,W_A + \alpha\,W_B$ for $\alpha \in [0,1]$. Stage-II tuning is identical for every $W_\alpha$.

\paragraph{Results.}
The cross-seed pair is farther apart in weight space than the cross-data pair ($\|\Delta W\|_F = 92.4$ versus $63.7$), but only the cross-data pair stays inside a low-loss region under interpolation. Cross-seed CIDEr~\cite{vedantam2015cider} on the LLaVA captioning evaluation collapses from approximately $80$ at the two endpoints to $0.7$ at $\alpha = 0.5$. Cross-data projectors interpolate smoothly along the entire path and the aggregate downstream score peaks at $\alpha = 0.5$ (Figure~\ref{fig:interp_heatmap}).

\paragraph{Interpretation.}
Cross-seed projectors reach distinct loss basins separated by high-loss barriers, which is the standard picture for independently initialized networks~\cite{frankle2020linear}. Cross-data projectors trained from the same initialization remain inside one shared low-loss region. Initialization therefore selects the basin, and training data selects the particular solution reached within that basin. This is the property that PORTAL targets directly: the closed-form construction lands inside the same compatible region without using paired data.

\clearpage
\section{Full Downstream Benchmarks}
\label{app:full_benchmarks}

\subsection{LCS-558k Pretraining}
\label{app:lcs558k}

\begin{table}[h]
\centering
\caption{\textbf{Per-benchmark results for each LCS-558k data fraction.}
Architecture: CLIP ViT-L/14 (336\,px) and MobileLLaMA~1.4B~\cite{chu2024mobilevlm}, with a two-layer MLP projector ($1024 \to 2048 \to 2048$).
Rows: fraction of the $558$k-sample LCS corpus used for Stage-I pretraining ($0\%$ corresponds to no Stage-I, random projector init).
Columns: $16$ benchmarks grouped into four categories adapted from the Cambrian-1 taxonomy (Appendix~\ref{app:eval_setup}), followed by the overall average and its ratio to the $100\%$ reference.
All scores are normalized to $0$--$100$ (MME is divided by $20$, others are already percentages).
$^\text{L}$: $500$-sample Lite subset from \texttt{lmms-eval}~\cite{zhang2024lmmseval}.}
\label{tab:full_lcs558k}
\vspace{2pt}
\resizebox{\textwidth}{!}{%
\setlength{\tabcolsep}{3.5pt}
\renewcommand{\arraystretch}{1.08}
\begin{tabular}{@{}l cccc G cccc G cccc G cccc G c c@{}}
\toprule
& \multicolumn{4}{c}{Knowledge} & \multicolumn{4}{c}{General} & \multicolumn{4}{c}{OCR} & \multicolumn{4}{c}{Vision} & & \\
\cmidrule(lr){2-5} \cmidrule(lr){6-9} \cmidrule(lr){10-13} \cmidrule(lr){14-17}
\textbf{Pretrain} & \rotatebox{60}{AI2D} & \rotatebox{60}{SciQA} & \rotatebox{60}{MathV} & \rotatebox{60}{MMMU} & \rotatebox{60}{GQA} & \rotatebox{60}{MME} & \rotatebox{60}{SEED$^\text{L}$} & \rotatebox{60}{MMB} & \rotatebox{60}{TxtVQA} & \rotatebox{60}{OCR-Bv2} & \rotatebox{60}{DocVQA} & \rotatebox{60}{ChartQA} & \rotatebox{60}{CV-2D} & \rotatebox{60}{CV-3D} & \rotatebox{60}{BLINK} & \rotatebox{60}{RealWQ} & \cellcolor{avgrow}\textbf{Avg} & \textbf{\%BL} \\
\midrule
\rowcolor{baserow}
0\%   & 34.5 & 47.4 & 25.2 & 26.6 & 51.2 & 49.7 & 41.6 & 10.1 & 17.0 & 7.0  & 8.9  & 10.3 & 36.6 & 51.7 & 35.3 & 41.6 & \cellcolor{avgrow}30.9 & 80.9 \\
10\%  & 37.5 & 53.0 & 24.3 & 28.0 & 56.5 & 56.4 & 58.6 & 39.3 & 32.1 & 7.8  & 13.1 & 10.6 & 43.7 & 53.2 & 35.7 & 44.4 & \cellcolor{avgrow}37.1 & 97.1 \\
\rowcolor{altrow}
20\%  & 39.1 & 55.4 & 24.4 & 27.6 & 56.8 & 58.7 & 59.0 & 39.0 & 33.6 & 8.3  & 14.2 & 11.3 & 44.6 & 51.2 & 35.5 & 44.2 & \cellcolor{avgrow}37.7 & 98.5 \\
30\%  & 39.1 & 53.7 & 25.8 & 29.3 & 56.6 & 57.6 & 58.6 & 37.7 & 34.8 & 7.6  & 14.9 & 10.6 & 44.6 & 51.8 & 35.7 & 44.3 & \cellcolor{avgrow}37.7 & 98.5 \\
\rowcolor{altrow}
40\%  & 39.2 & 55.2 & 25.9 & 29.1 & 56.7 & 58.2 & 59.2 & 41.8 & 35.2 & 8.1  & 15.0 & 10.9 & 45.6 & 52.9 & 34.8 & 44.6 & \cellcolor{avgrow}38.3 & 100.1 \\
50\%  & 39.7 & 54.9 & 25.8 & 29.2 & 57.1 & 58.2 & 59.6 & 42.2 & 35.5 & 8.2  & 15.4 & 10.9 & 46.0 & 53.2 & 35.6 & 46.3 & \cellcolor{avgrow}\cellcolor{bestcell}\textbf{38.6} & \cellcolor{bestcell}\textbf{101.0} \\
\rowcolor{altrow}
60\%  & 38.9 & 55.3 & 27.1 & 29.0 & 56.7 & 57.6 & 59.8 & 43.0 & 35.2 & 7.8  & 15.4 & 11.0 & 44.6 & 52.6 & 35.8 & 44.1 & \cellcolor{avgrow}38.4 & 100.3 \\
70\%  & 39.4 & 55.2 & 27.1 & 29.0 & 56.9 & 58.7 & 59.6 & 41.2 & 35.6 & 7.6  & 15.2 & 11.5 & 45.5 & 52.4 & 34.6 & 45.2 & \cellcolor{avgrow}\underline{38.4} & \underline{100.5} \\
\rowcolor{altrow}
80\%  & 39.2 & 54.7 & 26.2 & 28.7 & 57.0 & 57.2 & 59.8 & 41.2 & 35.6 & 8.8  & 15.4 & 10.9 & 44.6 & 53.7 & 34.6 & 44.1 & \cellcolor{avgrow}38.2 & 100.0 \\
90\%  & 39.4 & 55.2 & 26.4 & 28.8 & 57.4 & 57.0 & 59.4 & 41.4 & 35.7 & 7.8  & 15.5 & 11.1 & 46.2 & 54.0 & 34.7 & 44.2 & \cellcolor{avgrow}38.4 & 100.4 \\
\rowcolor{altrow}
100\% & 39.6 & 53.7 & 26.7 & 28.1 & 57.0 & 57.9 & 58.6 & 39.9 & 35.7 & 7.8  & 15.7 & 11.6 & 46.4 & 52.2 & 35.9 & 45.0 & \cellcolor{avgrow}38.2 & 100.0 \\
\bottomrule
\end{tabular}}
\end{table}

\subsection{ShareGPT-1.2M Pretraining}
\label{app:sharegpt12m}

\begin{table}[!htbp]
\centering
\caption{\textbf{Per-benchmark results for each ShareGPT-1.2M data fraction.}
Layout matches Table~\ref{tab:full_lcs558k} but with the ShareGPT-1.2M corpus used for Stage-I pretraining.}
\label{tab:full_sharegpt12m}
\vspace{2pt}
\resizebox{\textwidth}{!}{%
\setlength{\tabcolsep}{3.5pt}
\renewcommand{\arraystretch}{1.08}
\begin{tabular}{@{}l cccc G cccc G cccc G cccc G c c@{}}
\toprule
& \multicolumn{4}{c}{Knowledge} & \multicolumn{4}{c}{General} & \multicolumn{4}{c}{OCR} & \multicolumn{4}{c}{Vision} & & \\
\cmidrule(lr){2-5} \cmidrule(lr){6-9} \cmidrule(lr){10-13} \cmidrule(lr){14-17}
\textbf{Pretrain} & \rotatebox{60}{AI2D} & \rotatebox{60}{SciQA} & \rotatebox{60}{MathV} & \rotatebox{60}{MMMU} & \rotatebox{60}{GQA} & \rotatebox{60}{MME} & \rotatebox{60}{SEED$^\text{L}$} & \rotatebox{60}{MMB} & \rotatebox{60}{TxtVQA} & \rotatebox{60}{OCR-Bv2} & \rotatebox{60}{DocVQA} & \rotatebox{60}{ChartQA} & \rotatebox{60}{CV-2D} & \rotatebox{60}{CV-3D} & \rotatebox{60}{BLINK} & \rotatebox{60}{RealWQ} & \cellcolor{avgrow}\textbf{Avg} & \textbf{\%BL} \\
\midrule
\rowcolor{baserow}
0\%   & 34.5 & 47.4 & 25.2 & 26.6 & 51.2 & 49.7 & 41.6 & 10.1 & 17.0 & 7.0  & 8.9  & 10.3 & 36.6 & 51.7 & 35.3 & 41.6 & \cellcolor{avgrow}30.9 & 81.4 \\
10\%  & 38.4 & 53.7 & 24.3 & 28.1 & 55.1 & 54.5 & 53.5 & 50.0 & 29.6 & 6.9  & 12.5 & 11.2 & 40.8 & 52.4 & 35.2 & 43.1 & \cellcolor{avgrow}36.8 & 97.0 \\
\rowcolor{altrow}
20\%  & 37.9 & 54.2 & 24.1 & 28.1 & 55.7 & 55.1 & 55.4 & 54.5 & 30.7 & 6.5  & 12.7 & 11.1 & 41.7 & 53.8 & 32.3 & 44.1 & \cellcolor{avgrow}37.4 & 98.4 \\
30\%  & 38.7 & 53.9 & 24.4 & 28.3 & 56.4 & 55.7 & 56.4 & 51.5 & 31.1 & 6.8  & 13.1 & 11.2 & 40.5 & 52.9 & 35.6 & 45.2 & \cellcolor{avgrow}37.6 & 99.0 \\
\rowcolor{altrow}
40\%  & 38.0 & 53.9 & 24.8 & 28.4 & 56.4 & 57.5 & 57.2 & 50.8 & 31.4 & 6.8  & 13.2 & 11.6 & 44.4 & 52.8 & 35.3 & 46.4 & \cellcolor{avgrow}\cellcolor{bestcell}\textbf{38.1} & \cellcolor{bestcell}\textbf{100.2} \\
50\%  & 38.3 & 54.4 & 26.0 & 27.7 & 56.7 & 57.0 & 54.9 & 52.3 & 31.9 & 7.3  & 14.0 & 11.6 & 41.5 & 52.8 & 34.6 & 48.2 & \cellcolor{avgrow}\underline{38.1} & \underline{100.2} \\
\rowcolor{altrow}
60\%  & 38.4 & 53.7 & 26.4 & 28.7 & 56.9 & 56.4 & 56.4 & 52.3 & 32.8 & 7.7  & 13.5 & 11.5 & 43.2 & 53.8 & 31.1 & 46.4 & \cellcolor{avgrow}38.1 & 100.2 \\
70\%  & 38.3 & 54.4 & 26.8 & 28.4 & 57.1 & 56.3 & 55.6 & 50.8 & 33.9 & 7.8  & 14.1 & 11.4 & 45.0 & 52.2 & 30.9 & 45.4 & \cellcolor{avgrow}38.0 & 100.1 \\
\rowcolor{altrow}
80\%  & 38.7 & 53.7 & 26.0 & 28.6 & 57.1 & 57.3 & 55.4 & 49.2 & 34.1 & 7.9  & 14.1 & 11.5 & 41.9 & 52.2 & 30.9 & 46.8 & \cellcolor{avgrow}37.8 & 99.6 \\
90\%  & 38.2 & 54.9 & 26.2 & 29.2 & 57.1 & 57.1 & 57.2 & 50.0 & 34.1 & 8.2  & 14.4 & 11.1 & 43.0 & 51.9 & 24.2 & 45.1 & \cellcolor{avgrow}37.6 & 99.1 \\
\rowcolor{altrow}
100\% & 38.6 & 54.4 & 25.5 & 28.2 & 57.2 & 55.8 & 57.0 & 53.0 & 34.4 & 7.9  & 14.5 & 11.9 & 42.8 & 53.0 & 26.7 & 46.7 & \cellcolor{avgrow}38.0 & 100.0 \\
\bottomrule
\end{tabular}}
\end{table}

\clearpage
\section{Additional Experiments}
\label{app:additional_experiments}

In this section, we provide additional experiments that further evaluate the scope and robustness of PORTAL. 
We study whether the coarse-alignment phenomenon extends to a substantially higher-capacity connector, whether PORTAL remains effective under larger-scale multimodal instruction tuning, how its performance varies across task categories, how it compares with alternative unpaired alignment methods, and whether training-free alignment affects object hallucination.

\subsection{Comparison with Alternative Unpaired Alignment Methods}
\label{app:alignment_baselines}

Random initialization in our main experiments serves as a control for completely skipping Stage-I alignment rather than as the only competing alignment method.
We additionally compare PORTAL against CORAL~\cite{sun2016return}, which aligns second-order feature statistics in closed form, and Gromov-Wasserstein (GW) alignment~\cite{alvarez2018gromov}, an iterative unpaired alignment method.

\begin{table}[!htbp]
\centering
\small
\caption{
Comparison with alternative alignment methods.
All methods except the conventional Stage-I baseline require no paired image-caption samples.
}
\label{tab:alignment_baselines}
\begin{tabular}{lcc}
\toprule
\textbf{Initialization} & \textbf{Avg16} & \textbf{\% of Baseline} \\
\midrule
Stage-I trained baseline & 48.03 & 100.0 \\
Random                   & 44.90 & 93.5 \\
CORAL                    & 46.36 & 96.5 \\
Gromov-Wasserstein      & 47.06 & 98.0 \\
PORTAL                   & \textbf{47.43} & \textbf{98.8} \\
\bottomrule
\end{tabular}
\end{table}

\begin{table}[!htbp]
\centering
\small
\caption{
Characteristics of the unpaired initialization methods considered in Table~\ref{tab:alignment_baselines}.
}
\label{tab:alignment_baseline_properties}
\begin{tabular}{lccc}
\toprule
\textbf{Method} & \textbf{Paired Samples} &
\textbf{Projector Training} & \textbf{Construction} \\
\midrule
Random & No & No & Random initialization \\
CORAL  & No & No & Closed form \\
GW     & No & No & Iterative optimization \\
PORTAL & No & No & Closed form \\
\bottomrule
\end{tabular}
\end{table}

Both CORAL and GW substantially improve over random initialization, recovering 96.5\% and 98.0\% of the paired Stage-I baseline, respectively (Table \ref{tab:alignment_baselines}).
This confirms more broadly that useful projector initialization can be obtained from unpaired feature distributions.
PORTAL performs best among the unpaired methods considered here, recovering 98.8\% of the trained baseline while retaining a closed-form construction and requiring neither paired samples nor projector optimization.

\subsection{Performance Retention Across Task Categories}
\label{app:category_breakdown}

We next examine whether the aggregate performance of PORTAL hides systematic differences across task categories.
Table~\ref{tab:category_retention} reports performance relative to the corresponding conventionally Stage-I-trained baseline for Knowledge, General, OCR, and Vision-Centric benchmarks.

\begin{table}[!htbp]
\centering
\small
\caption{
PORTAL performance as a percentage of the corresponding Stage-I-trained baseline, broken down by task category.
The last row averages the retention rate across the seven (vision encoder, LLM) configurations.
}
\label{tab:category_retention}
\resizebox{\linewidth}{!}{
\begin{tabular}{llccccc}
\toprule
\textbf{LLM} & \textbf{Vision Encoder} &
\textbf{Knowledge} & \textbf{General} &
\textbf{OCR} & \textbf{Vision} & \textbf{Overall} \\
\midrule
Vicuna 1.5 7B  & CLIP   & 100.3 & 97.3  & 96.6  & 100.5 & 98.8 \\
Vicuna 1.5 7B  & SigLIP & 100.6 & 98.1  & 97.6  & 101.0 & 99.3 \\
LLaMA 3.1 8B   & CLIP   & 103.6 & 100.7 & 103.9 & 103.8 & 102.7 \\
LLaMA 3.2 1B   & CLIP   & 102.5 & 109.8 & 112.1 & 103.1 & 106.2 \\
Qwen 2.5 7B    & CLIP   & 98.5  & 99.4  & 92.6  & 99.3  & 98.3 \\
Qwen 2.5 0.5B  & CLIP   & 100.9 & 113.2 & 121.1 & 97.9  & 106.1 \\
OLMo 2 1B      & CLIP   & 98.2  & 95.4  & 92.9  & 97.4  & 96.3 \\
\midrule
\textbf{Mean across configurations} & --- &
\textbf{100.7} & \textbf{102.0} &
\textbf{102.4} & \textbf{100.4} & \textbf{101.1} \\
\bottomrule
\end{tabular}}
\end{table}

On average across configurations, PORTAL retains 100.7\%, 102.0\%, 102.4\%, and 100.4\% of the trained-baseline performance on Knowledge, General, OCR, and Vision-Centric tasks, respectively.
Thus, the aggregate performance of PORTAL is not driven by a single task category.
At the same time, paired Stage-I training remains beneficial for particular encoder-LLM combinations and individual benchmarks. For example, the largest category-level deficits occur for OCR under Qwen-2.5-7B and OLMo-2-1B.
The results therefore suggest a more nuanced picture than fine-grained tasks universally requiring paired Stage-I supervision: the value of paired supervision depends on the specific backbone, encoder, and benchmark.

\subsection{Large-Scale Multimodal Instruction Tuning}
\label{app:large_sft}

Our main experiments use LLaVA-665k for Stage-II instruction tuning.
We further test whether PORTAL remains effective under a substantially larger multimodal fine-tuning regime using 3.2M samples subsampled from Cambrian-7M.
This setting provides a stronger test of whether PORTAL initialization remains stable when followed by heavier multimodal training.

\begin{table}[!htbp]
\centering
\small
\caption{
Performance under large-scale Stage-II multimodal instruction tuning with 3.2M samples.
The confidence interval reports the 95\% CI for the difference between PORTAL and the Stage-I-trained baseline.
}
\label{tab:large_sft}
\begin{tabular}{lcccc}
\toprule
\textbf{Category} & \textbf{Baseline} & \textbf{PORTAL} & $\boldsymbol{\Delta}$ & \textbf{95\% CI} \\
\midrule
Knowledge & 45.95 & 46.68 & +0.72 & $[-0.40,\,+1.88]$ \\
General   & 59.60 & 59.88 & +0.27 & $[-0.71,\,+1.27]$ \\
OCR       & 30.28 & 31.45 & +1.18 & $[+0.67,\,+1.68]$ \\
Vision    & 50.12 & 51.28 & +1.16 & $[-0.13,\,+2.46]$ \\
\midrule
\textbf{Overall} & \textbf{46.49} & \textbf{47.32} &
\textbf{+0.83} & $\mathbf{[+0.36,\,+1.35]}$ \\
\bottomrule
\end{tabular}
\end{table}

PORTAL remains effective under this substantially heavier Stage-II training regime (Table~\ref{tab:large_sft}).
Across the 16 benchmarks, the overall mean increases from 46.49 with conventional Stage-I initialization to 47.32 with PORTAL, corresponding to a statistically significant improvement of $+0.83$ points with a 95\% confidence interval of $[+0.36,+1.35]$.
Importantly, PORTAL does not lead to performance collapse or degradation under large-scale multimodal fine-tuning.
The result suggests that the benefit of the closed-form initialization is not restricted to the standard LLaVA-665k Stage-II regime.

\subsection{Object Hallucination}
\label{app:pope}

Finally, we evaluate whether replacing paired Stage-I alignment with coarse statistical alignment increases object hallucination.
We evaluate all seven (vision encoder, LLM) configurations on POPE~\cite{li2023evaluating}, a standard benchmark for object hallucination (Table~\ref{tab:pope}).
PORTAL matches or outperforms the conventionally Stage-I-trained baseline across all seven configurations, with improvements ranging from $+0.2$ to $+2.2$ F1 points.
Thus, replacing paired Stage-I projector training with PORTAL does not increase object hallucination in these experiments and can slightly improve hallucination robustness.

\begin{table}[!htbp]
\centering
\small
\caption{
POPE F1 scores across all seven (vision encoder, LLM) configurations.
Numbers in parentheses indicate the improvement of PORTAL over the conventionally Stage-I-trained baseline.
}
\label{tab:pope}
\begin{tabular}{llccc}
\toprule
\textbf{LLM} & \textbf{Vision Encoder} &
\textbf{Baseline} & \textbf{Random} & \textbf{PORTAL} \\
\midrule
Vicuna 1.5 7B & CLIP
& 85.5 & 85.3 & \textbf{85.7} $(+0.2)$ \\
Qwen 2.5 7B & CLIP
& 85.7 & 85.5 & \textbf{86.4} $(+0.7)$ \\
LLaMA 3.1 8B & CLIP
& 84.0 & 84.2 & \textbf{84.9} $(+0.9)$ \\
Vicuna 1.5 7B & SigLIP
& 84.6 & 79.1 & \textbf{85.1} $(+0.5)$ \\
OLMo 2 1B & CLIP
& 85.3 & 84.8 & \textbf{85.5} $(+0.2)$ \\
LLaMA 3.2 1B & CLIP
& 84.1 & 84.0 & \textbf{86.2} $(+2.1)$ \\
Qwen 2.5 0.5B & CLIP
& 82.9 & 82.7 & \textbf{85.1} $(+2.2)$ \\
\bottomrule
\end{tabular}
\end{table}

\subsection{Higher-Capacity Connector}
\label{app:cabstractor}

Our main experiments use either a two-layer MLP or a linear projector.
A possible alternative explanation for the observed coarse-alignment behavior is therefore that it arises from limited projector capacity.
To test this, we additionally evaluate PORTAL with the C-Abstractor connector from Honeybee~\cite{cha2024honeybee}, a substantially higher-capacity architecture that performs both token mixing and sequence compression.

In the LLaMA-3.2-1B + CLIP setting, C-Abstractor contains 120.2M parameters, approximately $19\times$ more than the 6.3M-parameter two-layer MLP, and compresses the visual sequence from 576 to 144 tokens.
We initialize the residual branches of C-Abstractor to zero such that the connector initially reduces to applying PORTAL to the pooled visual features.
All remaining C-Abstractor parameters are then optimized during the same Stage-II training used for the other initialization schemes.

As shown in Table~\ref{tab:cabstractor}, PORTAL reaches an average score of 42.93, compared with 42.98 for conventional Stage-I training, while direct Stage-II training from random initialization reaches only 41.46.
PORTAL also outperforms the Stage-I-trained baseline on 7 of the 16 individual benchmarks.
Thus, increasing connector capacity by approximately $19\times$, while additionally enabling token mixing and sequence compression, does not yield an aggregate advantage for paired Stage-I training over PORTAL.
Together with our linear and two-layer MLP experiments, this result provides evidence that the observed coarse-alignment behavior cannot be explained simply by insufficient projector capacity.
We do not, however, claim that every possible connector architecture must exhibit identical behavior.

\begin{table}[!htbp]
\centering
\small
\caption{
Performance with the higher-capacity C-Abstractor connector.
Despite substantially increasing connector capacity and introducing explicit token mixing and sequence compression, PORTAL remains essentially indistinguishable from conventional Stage-I training.
}
\label{tab:cabstractor}
\begin{tabular}{lccc}
\toprule
\textbf{Task} & \textbf{Baseline} & \textbf{Random} & \textbf{PORTAL} \\
\midrule
AI2D            & 49.3 & 46.7 & 48.0 \\
ScienceQA       & 68.2 & 65.2 & 66.2 \\
MathVista       & 26.9 & 25.9 & 26.2 \\
MMMU            & 32.3 & 34.6 & 33.9 \\
GQA             & 60.6 & 59.7 & 60.7 \\
MME-P           & 60.8 & 59.6 & 65.4 \\
SEED-L          & 58.4 & 54.7 & 56.2 \\
MMBench         & 58.4 & 52.4 & 54.9 \\
TextVQA-OCR     & 46.7 & 43.6 & 45.4 \\
OCRBench-v2     & 10.7 & 11.2 & 11.5 \\
DocVQA          & 13.2 & 11.8 & 12.7 \\
ChartQA         & 11.1 & 10.4 & 10.8 \\
CV-Bench 2D     & 50.1 & 47.8 & 50.8 \\
CV-Bench 3D     & 54.2 & 53.9 & 55.2 \\
BLINK           & 38.3 & 37.0 & 40.9 \\
RealWorldQA     & 48.6 & 48.9 & 48.1 \\
\midrule
\textbf{Average} & \textbf{42.98} & \textbf{41.46} & \textbf{42.93} \\
\bottomrule
\end{tabular}
\end{table}

\clearpage
\section{PORTAL Construction: Pseudocode}
\label{app:portal_pseudocode}

Listing~\ref{lst:portal} states the PORTAL construction as a self-contained PyTorch routine. The function consumes a sample of source features (\textit{e.g.}, CLIP patch tokens), a sample of target features (\textit{e.g.}, the LLM's input embedding table or the hidden states after block $\ell$), and a rank $k$. It returns the parameters $(\mathbf{W}_1, \mathbf{b}_1, \mathbf{W}_2, \mathbf{b}_2)$ of the standard two-layer LLaVA projector $f(\mathbf{x}) = \mathbf{W}_2\,\phi(\mathbf{W}_1\mathbf{x} + \mathbf{b}_1) + \mathbf{b}_2$ from Section~\ref{sec:method_construction}, with $\phi$ a GELU. The first layer is the closed-form Gaussian optimal-transport map of Eq.~\eqref{eq:wot}, and the second layer is fixed to the identity. No paired captions and no Stage-I optimization are required, only one PCA on each side.

\begin{lstlisting}[style=neuriPy,caption={PORTAL: training-free projector via Gaussian optimal transport. The same routine yields PORTAL-ctx$\ell$ when $Y$ is replaced by post-block-$\ell$ hidden states (Appendix~\ref{app:ctx_methodology}).},label={lst:portal}]
import torch

def portal(X, Y, k=512):
    """
    X: source features, shape (N_s, d_s)  -- e.g. CLIP patch tokens.
    Y: target features, shape (N_t, d_t)  -- LLM input embeddings,
                                             or post-block-l hidden
                                             states for PORTAL-ctx-l.
    k: PCA rank for the shared subspace.

    Returns (W1, b1, W2, b2) for the LLaVA projector
        f(x) = W2 @ gelu(W1 @ x + b1) + b2.
    """
    # 1. Source statistics: top-k PCA on centered source features.
    mu_s        = X.mean(dim=0)                              # (d_s,)
    _, S_s, Vt_s = torch.linalg.svd(X - mu_s, full_matrices=False)
    V_s         = Vt_s[:k].T                                 # (d_s, k)
    sigma_s     = S_s[:k] / (X.shape[0] - 1) ** 0.5          # (k,)

    # 2. Target statistics: top-k PCA on centered target features.
    mu_t        = Y.mean(dim=0)                              # (d_t,)
    _, S_t, Vt_t = torch.linalg.svd(Y - mu_t, full_matrices=False)
    V_t         = Vt_t[:k].T                                 # (d_t, k)
    sigma_t     = S_t[:k] / (Y.shape[0] - 1) ** 0.5          # (k,)

    # 3. Closed-form projector (Eq. 2): W2 fixed to identity.
    W1 = V_t @ torch.diag(sigma_t / sigma_s) @ V_s.T         # (d_t, d_s)
    b1 = mu_t - W1 @ mu_s                                    # (d_t,)
    W2 = torch.eye(Y.shape[1])
    b2 = torch.zeros(Y.shape[1])
    return W1, b1, W2, b2
\end{lstlisting}

\paragraph{Inputs.} The source matrix \texttt{X} contains all per-patch tokens from a small unlabeled image pool: with $P$ patches per image and $N_{\mathrm{img}}$ images, $N_s = N_{\mathrm{img}} \cdot P$. Throughout the paper we use a default of $N_{\mathrm{img}}=1024$ images from COCO \texttt{val2017} (Appendix~\ref{app:image_pool_ablation}). For default PORTAL, the target matrix \texttt{Y} is the LLM's input embedding table, with one row per vocabulary token. For PORTAL-ctx$\ell$, \texttt{Y} instead holds the hidden states after block $\ell$, which we collect by passing a list of English words and short phrases (the \texttt{concepts.txt} file of LatentLens) through the LLM (Appendix~\ref{app:ctx_methodology}).

\paragraph{Interpretation.} Read right to left, \texttt{V\_s.T} projects each centered source token onto the source's top-$k$ principal axes, the diagonal rescaling \texttt{diag(sigma\_t / sigma\_s)} rescales each axis from source to target spread, and \texttt{V\_t} rotates the result into the target's top-$k$ principal frame. Adding \texttt{b1 = mu\_t - W1 @ mu\_s} then recenters to the target mean. The composition matches first and second moments of source and target on the $k$-dimensional shared subspace. We use $k{=}512$ as the default (Appendix~\ref{app:rank_ablation}). The construction is well-defined for any $k \le \min(d_s, d_t, N_s, N_t)$.

\paragraph{Cost.} The procedure runs two thin SVDs and two matrix multiplies. For the default setting ($N_s \!\approx\! 1024 \cdot 576$, $d_s = 1024$, $d_t = 4096$, $k = 512$) the construction completes in under $30$ seconds on a single RTX~3090 (Appendix~\ref{app:compute}) and uses no captions, no paired data, and no gradient steps. This stands in contrast to standard Stage-I projector training, which consumes a paired image-caption corpus (typically LCS-558k) and about $2{,}200$ optimizer steps ($558$k pairs at batch size $256$).

\section{Contextualized Targets: Methodology Details}
\label{app:ctx_methodology}

This appendix provides the implementation details behind PORTAL-ctx$\ell$ (Section~\ref{sec:method_construction}).

\paragraph{Probe corpus.}
We use the \texttt{concepts.txt} corpus released with LatentLens~\citep{latentlens}, a list of English concepts originally assembled for probing latent directions in LLMs. The corpus offers broad semantic coverage and is not tied to any downstream task, which makes it a representative sample of the input distribution of the LLM at depth $\ell$.

\paragraph{Cached hidden states.}
For LLMs covered by LatentLens (LLaMA 3.1 8B, Qwen 2.5 7B, Gemma 2 9B) we use the pre-computed hidden states distributed on Hugging Face. For each LLM and each layer $\ell$, the cache stores the post-block-$\ell$ hidden state of every token in every \texttt{concepts.txt} entry, collected in a single tensor $\mathbf H^{(\ell)} \in \mathbb R^{N_t \times d_t}$. The available layers are $\ell \in \{1, 4, 8, 12, 16, 20\}$. For Vicuna~v1.5, which is not part of the LatentLens release, we run \texttt{concepts.txt} through the model and store the corresponding hidden states in the same format. The tensor $\mathbf H^{(\ell)}$ replaces the LLM's input embedding table as the target distribution in Eq.~\eqref{eq:wot}. The remainder of the PORTAL construction is unchanged.

\paragraph{Layer indexing.}
$\ell$ denotes the residual stream \emph{after} the $\ell$-th transformer block (1-indexed), following the LatentLens convention. The setting $\ell{=}0$ recovers standard PORTAL targeting the LLM's input embedding table.

\section{Cosine Dispersion: Theory and Empirical Checks}
\label{app:cosine_dispersion}

This section consolidates the analysis underlying Section~\ref{sec:method_uniformity}. It quantifies how non-Gaussian the relevant source and target distributions are (Appendix~\ref{app:gaussianity_audit}), proves the spectrum-to-$\omega$ bound (Appendix~\ref{app:concentration}), reports the empirical check of that bound on PORTAL projector outputs (Appendix~\ref{app:sigma_c_projector}), and extends the empirical check to the full MLP-GELU projector (Appendix~\ref{app:gelu_diagnostic}).

\subsection{Distributional Audit of Source and Target Modalities}
\label{app:gaussianity_audit}

PORTAL treats both modalities as Gaussian in Eq.~\eqref{eq:wot} only to obtain a closed-form transport map. The empirical source and target distributions are not exactly Gaussian. This subsection quantifies the deviation so that the reader can assess how reliable the second-order summary used by PORTAL is.

\paragraph{Procedure.}
For each distribution we draw $n{=}5{,}000$ tokens (CLIP and SigLIP patches from COCO \texttt{val2017} images, rows of the LLM's input embedding table), project the centered features onto their top $k{=}128$ principal components and whiten the resulting coordinates to unit per-axis variance. On these coordinates we compute four statistics. Mardia's multivariate skewness $b_1$~\cite{mardia1970}, reported as $b_1/k^3$, which is close to zero under Gaussianity. Mardia's kurtosis $b_2$~\cite{mardia1970}, reported as the excess $b_2/k^2 - (1 + 2/k)$ over its Gaussian value, which quantifies tail weight relative to a Gaussian. The fraction of the $k$ whitened marginals for which the Anderson-Darling test~\cite{anderson1952} rejects $\mathcal{N}(0,1)$ at the $5\%$ level. The effective dimension $d_{\mathrm{eff}} = (\mathrm{tr}\,\boldsymbol\Sigma)^2 / \mathrm{tr}(\boldsymbol\Sigma^2)$~\cite{roy2007effective}, which summarizes how variance is distributed across the leading directions. Results are reported in Table~\ref{tab:gaussianity_audit}.

\begin{table}[!htbp]
\centering
\caption{\textbf{Distributional audit on top $k{=}128$ principal components.} The columns most relevant to the argument in Section~\ref{sec:method_construction} are $d_{\mathrm{eff}}$ and Mardia $b_1$. The Anderson-Darling rejection rate (AD rej.) and the excess kurtosis identify where the deviation from Gaussianity is concentrated.}
\label{tab:gaussianity_audit}
\vspace{2pt}
\small
\setlength{\tabcolsep}{4pt}
\renewcommand{\arraystretch}{1.05}
\begin{tabular}{l l r r r r r r}
\toprule
\textbf{Type} & \textbf{Distribution} & $n$ & $d$ & $d_{\mathrm{eff}}$ & Mardia $b_1$ & Exc.\ kurt. & AD rej. \\
\midrule
\rowcolor{baserow} Source & CLIP ViT-L/14 patches & 5{,}000 & 1{,}024 & 4.5 & 0.0022 & 0.487 & 100\% \\
\rowcolor{baserow} Source & SigLIP SO400M patches & 5{,}000 & 1{,}152 & 50.9 & 0.0026 & 0.278 & 100\% \\
\midrule
Target & Vicuna 7B embedding table & 5{,}000 & 4{,}096 & 1{,}429.6 & 0.0005 & 0.071 & 84\% \\
Target & LLaMA 3.1 8B embedding table & 5{,}000 & 4{,}096 & 1{,}713.2 & 0.0005 & 0.131 & 98\% \\
Target & Qwen 2.5 7B embedding table & 5{,}000 & 3{,}584 & 830.1 & 0.0004 & 0.114 & 100\% \\
\bottomrule
\end{tabular}
\end{table}

\paragraph{Results.}
No measured distribution is exactly Gaussian, and the Anderson-Darling test rejects Gaussianity on almost all top marginals. This is expected for high-dimensional learned features and does not, on its own, imply that PORTAL should fail.

The question is where the deviation lives. Skewness is small for every distribution ($b_1/k^3 \le 0.003$), so the residual deviation from Gaussianity lies in tail weight. The tails are heaviest for the vision features (excess kurtosis $0.28$ to $0.49$) and light for the LLM input embedding tables ($0.07$ to $0.13$).

A second question is whether the distributions that PORTAL targets spread their variance over many directions. The input embedding tables do, with $d_{\mathrm{eff}}$ between $830$ and $1{,}713$, while the CLIP patches concentrate their variance in a few directions ($d_{\mathrm{eff}} = 4.5$).

\subsection{Proof of Proposition~\ref{prop:sigma_c}}
\label{app:concentration}

Proposition~\ref{prop:sigma_c} states that the variance of the cosine between two random tokens is close to $1/d_{\mathrm{eff}}(\boldsymbol\Sigma)$ when token lengths vary little. The proof first computes the second moment of the plain dot product $\mathbf z^\top \mathbf z'$, which depends only on $\boldsymbol\Sigma$, and then bounds how much dividing by the two token lengths can change it.

\paragraph{Notation.} The trace $\mathrm{tr}\,\boldsymbol\Sigma$ equals the mean squared length $\mathbb{E}[\|\mathbf z\|^2]$, because $\mathbf z$ has mean zero. We write
\begin{itemize}[leftmargin=1.5em,itemsep=1pt,topsep=2pt]
    \item $\mathbf u = \mathbf z/\|\mathbf z\|$ and $\mathbf u' = \mathbf z'/\|\mathbf z'\|$ for the unit vectors along $\mathbf z$ and $\mathbf z'$,
    \item $c = \mathbf u^\top \mathbf u' = \cos\angle(\mathbf z, \mathbf z')$ for the cosine,
    \item $a = \|\mathbf z\|^2/\mathrm{tr}\,\boldsymbol\Sigma$ and $a' = \|\mathbf z'\|^2/\mathrm{tr}\,\boldsymbol\Sigma$ for the squared lengths relative to their mean, so that $\mathbb{E}[a] = 1$ and $\mathrm{std}(a) = \kappa$,
    \item $m = \mathbb{E}|1 - a|$ for the mean deviation of $a$ from $1$. By the Cauchy-Schwarz inequality, $m \le \sqrt{\mathbb{E}[(1-a)^2]} = \kappa$.
\end{itemize}

\begin{proof}
\textit{Step 1: the dot product.} Averaging over $\mathbf z$ with $\mathbf z'$ held fixed, and using $\mathbb{E}[\mathbf z \mathbf z^\top] = \boldsymbol\Sigma$,
\begin{equation*}
\mathbb{E}\big[(\mathbf z^\top \mathbf z')^2\big] = \mathbb{E}\big[\mathbf z'^\top \mathbf z\, \mathbf z^\top \mathbf z'\big] = \mathbb{E}\big[\mathbf z'^\top \boldsymbol\Sigma\, \mathbf z'\big] = \mathrm{tr}\big(\boldsymbol\Sigma\, \mathbb{E}[\mathbf z' \mathbf z'^\top]\big) = \mathrm{tr}(\boldsymbol\Sigma^2).
\end{equation*}
For the normalized squared dot product $q = (\mathbf z^\top \mathbf z')^2/(\mathrm{tr}\,\boldsymbol\Sigma)^2$ this gives
\begin{equation*}
\mathbb{E}[q] = \frac{\mathrm{tr}(\boldsymbol\Sigma^2)}{(\mathrm{tr}\,\boldsymbol\Sigma)^2} = \frac{1}{d_{\mathrm{eff}}(\boldsymbol\Sigma)}.
\end{equation*}
This step is exact and holds for any distribution.

\textit{Step 2: from the dot product to the cosine.} The cosine is the dot product divided by both lengths, so $c^2 = q/(a a')$, or $q = c^2 a a'$. Hence
\begin{equation*}
\mathbb{E}[c^2] - \frac{1}{d_{\mathrm{eff}}(\boldsymbol\Sigma)} = \mathbb{E}[c^2 - q] = \mathbb{E}\big[c^2 (1 - a a')\big].
\end{equation*}
Since $c^2 \le 1$, the absolute value of the right side is at most $\mathbb{E}|1 - a a'|$. Writing $1 - aa' = (1-a) + (1-a') - (1-a)(1-a')$, the triangle inequality and the independence of $a$ and $a'$ give
\begin{equation*}
\mathbb{E}|1 - a a'| \le \mathbb{E}|1-a| + \mathbb{E}|1-a'| + \mathbb{E}|1-a|\,\mathbb{E}|1-a'| = 2m + m^2 .
\end{equation*}

\textit{Step 3: the mean cosine.} The variance subtracts the squared mean, $\mathrm{Var}(c) = \mathbb{E}[c^2] - \mathbb{E}[c]^2$, so it remains to bound $\mathbb{E}[c]$, the average cosine between two independent tokens. By independence, $\mathbb{E}[c] = \mathbb{E}[\mathbf u]^\top \mathbb{E}[\mathbf u'] = \|\mathbb{E}[\mathbf u]\|^2$. Because $\mathbf z = \sqrt{\mathrm{tr}\,\boldsymbol\Sigma}\,\sqrt{a}\,\mathbf u$ has mean zero, $\mathbb{E}[\sqrt a\, \mathbf u] = \mathbf 0$, and therefore
\begin{equation*}
\mathbb{E}[\mathbf u] = \mathbb{E}\big[(1 - \sqrt a)\, \mathbf u\big],
\qquad
\|\mathbb{E}[\mathbf u]\| \le \mathbb{E}|1 - \sqrt a| \le \mathbb{E}|1 - a| = m ,
\end{equation*}
where we used $\|\mathbf u\| = 1$ and $|1 - \sqrt a| = |1 - a|/(1 + \sqrt a) \le |1-a|$. Hence $\mathbb{E}[c]^2 = \|\mathbb{E}[\mathbf u]\|^4 \le m^4$.

\textit{Combining the steps.} By Steps 2 and 3 and $m \le \kappa$,
\begin{equation*}
\left| \mathrm{Var}(c) - \frac{1}{d_{\mathrm{eff}}(\boldsymbol\Sigma)} \right|
\le \left| \mathbb{E}[c^2] - \frac{1}{d_{\mathrm{eff}}(\boldsymbol\Sigma)} \right| + \mathbb{E}[c]^2
\le 2m + m^2 + m^4
\le 2\kappa + \kappa^2 + \kappa^4 . \qedhere
\end{equation*}
\end{proof}

The quantity $\omega^2$ in Eq.~\eqref{eq:sigma_c} is the empirical variance of the cosines over all pairs of $n$ sampled tokens. Both of its pairwise averages are averages of a bounded function over all pairs (U-statistics), so $\omega^2$ converges to $\mathrm{Var}(c)$ as $n$ grows.

\subsection{Cosine Dispersion of PORTAL Projector Outputs}
\label{app:sigma_c_projector}

This appendix measures $\omega$ on the PORTAL outputs for the three CLIP pairs with a 7B or 8B LLM in Table~\ref{tab:main_results}. We construct PORTAL as in Appendix~\ref{app:portal_pseudocode} from the first $1024$ COCO \texttt{val2017} images, apply it to the next $100$ images ($576$ CLIP patches per image), and report two quantities. The intra-image $\omega$ is averaged across images, with each image contributing the dispersion of its own $576$ output tokens. The population $\omega$ is computed over $576$ output tokens sampled at random across all images. The intra-image quantity reflects the dispersion seen by attention on a single forward pass. The population quantity is the distributional property that Proposition~\ref{prop:sigma_c} describes.

\begin{table}[!htbp]
\centering
\caption{\textbf{Cosine dispersion $\omega$ of projector outputs and downstream average.} For each (encoder, LLM) pair, \emph{Source} is $\omega$ of the raw source patches, \emph{Random} and \emph{PORTAL} are $\omega$ after applying the corresponding projector, and \emph{Target} is $\omega$ of the LLM's input embedding table restricted to the top-$512$ principal subspace that PORTAL matches. Random and PORTAL denote the first layer $\mathbf{W}_1\mathbf{x}+\mathbf{b}_1$ of each projector. \emph{PORTAL $d_{\mathrm{eff}}$} reports the effective dimension of the PORTAL output cloud. Downstream Avg columns from Table~\ref{tab:main_results}.}
\label{tab:sigma_c_projector}
\vspace{2pt}
\resizebox{\textwidth}{!}{%
\setlength{\tabcolsep}{4pt}
\renewcommand{\arraystretch}{1.05}
\begin{tabular}{l l c r r r r r r r r}
\toprule
\textbf{Encoder} & \textbf{LLM} & \textbf{Scale} & \textbf{Source} & \textbf{Random} & \textbf{PORTAL} & \textbf{Target} & \textbf{PORTAL $d_{\mathrm{eff}}$} & \textbf{Random Avg} & \textbf{PORTAL Avg} & \textbf{Trained Avg} \\
\midrule
\multirow{2}{*}{CLIP} & \multirow{2}{*}{Vicuna 7B}
   & pop   & 0.234 & 0.234 & \cellcolor{bestcell}\textbf{0.100} & 0.056 & 266.7 & 44.9 & \textbf{47.4} & 48.0 \\
 & & intra & 0.252 & 0.251 & 0.235 & 0.056 & 266.7 & 44.9 & \textbf{47.4} & 48.0 \\
\midrule
\multirow{2}{*}{CLIP} & \multirow{2}{*}{LLaMA 3.1 8B}
   & pop   & 0.234 & 0.234 & \cellcolor{bestcell}\textbf{0.101} & 0.057 & 257.4 & 50.4 & \textbf{52.2} & 50.8 \\
 & & intra & 0.252 & 0.251 & 0.234 & 0.057 & 257.4 & 50.4 & \textbf{52.2} & 50.8 \\
\midrule
\multirow{2}{*}{CLIP} & \multirow{2}{*}{Qwen 2.5 7B}
   & pop   & 0.234 & 0.234 & \cellcolor{bestcell}\textbf{0.094} & 0.054 & 289.3 & 48.6 & \textbf{51.7} & 52.6 \\
 & & intra & 0.252 & 0.251 & 0.245 & 0.054 & 289.3 & 48.6 & \textbf{51.7} & 52.6 \\
\bottomrule
\end{tabular}}
\end{table}

PORTAL reduces $\omega$ from the source value of $0.23$ to between $0.09$ and $0.10$, towards the target value of $0.05$ to $0.06$. Over the full embedding table, the target's $\omega$ is $0.02$ to $0.03$. A random projector roughly preserves angles and therefore leaves $\omega$ at the source value. The intra-image $\omega$ does not separate the three projectors, because patches from a single image share a common offset that no global linear map can remove.

The measurement is correlative. PORTAL moves $\omega$ most of the way from the source towards the target, and we do not claim that $\omega$ is the sole geometric property responsible for the downstream performance of PORTAL.

\subsection{Post-Projection Statistics: Linear vs.\ MLP-GELU Output}
\label{app:gelu_diagnostic}

The closed-form construction in Eq.~\eqref{eq:wot} gives an affine map $\mathbf{W}_1\mathbf{x}+\mathbf{b}_1$ that matches the target's first and second moments under a Gaussian approximation. The implemented projector then passes that affine output through GELU. This subsection reports distributional statistics of the actual projector output for each (encoder, LLM) pair across four constructions. \textit{Random+GELU} is a Random Kaiming initialization followed by GELU. \textit{PORTAL (preact)} is the PORTAL preactivation alone, which is the linear object the theory in Section~\ref{sec:method_uniformity} applies to. \textit{PORTAL+GELU} is the PORTAL preactivation followed by GELU, which is the implemented PORTAL projector. \textit{Trained Stage-I} is a trained Stage-I projector loaded from checkpoint. The target row reports the same statistics on the LLM's input embedding table, restricted to the top-$512$ principal subspace that PORTAL matches. \textit{Trained Stage-I} is the public LLaVA-1.5 projector for Vicuna and our trained projectors for LLaMA and Qwen. For the latter two we report only the statistics that do not depend on the target.

For each output cloud we report five statistics. The mean offset $\|\Delta\boldsymbol\mu\|$ is the $L^2$ distance between the output mean and the target mean, with a relative version normalized by $\|\boldsymbol\mu_t\|$. The spectrum error is the relative $L^2$ distance between the top-$128$ output and target singular values. The effective dimension is $d_{\mathrm{eff}}=(\mathrm{tr}\,\boldsymbol\Sigma)^2/\mathrm{tr}(\boldsymbol\Sigma^2)$. The norm coefficient of variation $\mathrm{cv} = \mathrm{std}(\|\mathbf{z}\|)/\mathbb{E}[\|\mathbf{z}\|]$ serves as a concentration check. The population pairwise-cosine standard deviation is $\omega$ as defined in Eq.~\eqref{eq:sigma_c}. Statistics are computed over $100$ COCO \texttt{val2017} images of $576$ CLIP patches each, with PORTAL constructed from a disjoint set of $1024$ images.

\begin{table}[!htbp]
\centering
\caption{\textbf{Post-projection statistics on CLIP patch tokens.} Distributional fit to the target distribution for four projector constructions, plus a target row for reference. Smaller $\|\Delta\boldsymbol\mu\|$ and spectrum error indicate a closer fit to the target. The $d_{\mathrm{eff}}$ and $\omega$ values should be read against the target row of the same LLM.}
\label{tab:gelu_diagnostic}
\vspace{2pt}
\small
\setlength{\tabcolsep}{4pt}
\renewcommand{\arraystretch}{1.05}
\begin{tabular}{@{}l l r r r r r r@{}}
\toprule
\textbf{LLM} & \textbf{Method} & $\|\Delta\boldsymbol\mu\|$ & \textbf{Rel.} & \textbf{Spec.\ err.} & $d_{\mathrm{eff}}$ & $\mathrm{cv}$ & $\omega$ (pop) \\
\midrule
Vicuna 7B & \textit{target} & 0.000 & 0.000 & 0.000 & 319.7 & 0.123 & 0.056 \\
 & Random+GELU & 7.144 & 47.33 & 24.25 & 3.6 & 0.682 & 0.212 \\
 & PORTAL (preact) & 0.021 & 0.142 & 0.066 & 266.7 & 0.281 & 0.100 \\
 & PORTAL+GELU & 0.077 & 0.508 & 0.470 & 266.9 & 0.281 & 0.100 \\
 & Trained Stage-I & 33.18 & 219.8 & 142.4 & 2.4 & 1.047 & 0.232 \\
\midrule
LLaMA 3.1 8B & \textit{target} & 0.000 & 0.000 & 0.000 & 308.2 & 0.135 & 0.057 \\
 & Random+GELU & 7.145 & 83.28 & 47.07 & 3.6 & 0.682 & 0.212 \\
 & PORTAL (preact) & 0.011 & 0.133 & 0.088 & 257.4 & 0.279 & 0.101 \\
 & PORTAL+GELU & 0.043 & 0.505 & 0.460 & 257.3 & 0.279 & 0.101 \\
 & Trained Stage-I & -- & -- & -- & 2.4 & 1.019 & 0.239 \\
\midrule
Qwen 2.5 7B & \textit{target} & 0.000 & 0.000 & 0.000 & 352.5 & 0.216 & 0.054 \\
 & Random+GELU & 6.791 & 81.61 & 25.55 & 3.6 & 0.686 & 0.212 \\
 & PORTAL (preact) & 0.020 & 0.242 & 0.097 & 289.3 & 0.294 & 0.094 \\
 & PORTAL+GELU & 0.043 & 0.511 & 0.456 & 289.1 & 0.294 & 0.094 \\
 & Trained Stage-I & -- & -- & -- & 2.0 & 1.147 & 0.221 \\
\bottomrule
\end{tabular}
\end{table}

The table supports four observations. First, Random+GELU is far from the target on every column, consistent with the SigLIP collapse reported in Table~\ref{tab:main_results}. Second, the PORTAL preactivation matches the target mean and top-$128$ spectrum, comes close to its $d_{\mathrm{eff}}$, and moves $\omega$ from $0.23$ at the source to about $0.10$, against $0.05$ to $0.06$ for the target. Third, applying GELU to the PORTAL preactivation roughly preserves $d_{\mathrm{eff}}$ and $\omega$, even though the spectrum error grows by an order of magnitude. This empirical observation closes the gap between the theory of the linear object and the implemented MLP. Fourth, a trained Stage-I projector deviates substantially from the target distribution on every column we report. Despite this, trained Stage-I and PORTAL reach comparable downstream Avg in Table~\ref{tab:main_results}. The two solutions are distributionally distinct, which is consistent with the broad-solution-space rather than with the view that Stage-I learns a unique target distribution.

\section{Image Pool for PORTAL Construction}
\label{app:image_pool_ablation}

PORTAL uses an unlabeled image pool only to estimate the CLIP patch-feature statistics required by Eq.~\eqref{eq:wot}, and requires neither captions nor paired text. We assess the sensitivity of the constructed projector to the identity and size of this pool. Table~\ref{tab:image_pool_ablation} reports an ablation on LLaMA~3.2~1B/CLIP along two axes.
\begin{itemize}
    \item \textbf{Source diversity.} COCO (\texttt{val2017}, single domain) vs.\ a reservoir-sampled subset of Cambrian-737k~\cite{tong2024cambrian} spanning nine sources (\texttt{coco}, \texttt{vg}, \texttt{ocr\_vqa}, \texttt{gqa}, \texttt{chartqa}, \texttt{textvqa}, \texttt{ai2d}, \texttt{docvqa}, \texttt{dvqa}). For reference, the trained baseline uses the LCS-558k~\cite{liu2024llava} caption pool.
    \item \textbf{Pool size.} $1024$ images (standard) vs.\ $4096$ images.
\end{itemize}

The smaller single-domain COCO pool ($1024$ images) matches or exceeds the larger Cambrian pool ($4096$ images) on $12$ of $15$ benchmarks, with $2$ exact ties and $3$ marginal Cambrian wins (AI2D, MMMU, DocVQA). The COCO pool yields an aggregate gain of $+0.6$ Avg ($44.0$ vs.\ $43.4$) and $+1.3$ \%BL ($106.1$ vs.\ $104.8$).

These numbers indicate that the second-order statistics of a generic natural-image distribution are sufficient to specify the whitening map. Increasing source diversity provides no additional benefit once the statistics have converged, in agreement with the rapid decay of the leading singular values of a CLIP feature matrix. On this configuration, PORTAL is insensitive to the choice between the two pools.

\begin{table}[!htbp]
\centering
\caption{\textbf{Image-pool ablation for the PORTAL construction, per-benchmark breakdown.} Setup: LLaMA~3.2~1B/CLIP, Stage-II on LLaVA-665k. Both PORTAL variants use rank $k{=}512$ and require no paired captions. \textbf{Bold}: best between the two PORTAL rows per benchmark. Avg and \%BL are computed over the $15$ benchmarks shown.}
\label{tab:image_pool_ablation}
\vspace{2pt}
\resizebox{\textwidth}{!}{%
\setlength{\tabcolsep}{3.5pt}
\renewcommand{\arraystretch}{1.08}
\begin{tabular}{@{}l c c G cccc G cccc G ccc G cccc G c c@{}}
\toprule
& & & \multicolumn{4}{c}{Knowledge} & \multicolumn{4}{c}{General} & \multicolumn{3}{c}{OCR} & \multicolumn{4}{c}{Vision} & & \\
\cmidrule(lr){4-7} \cmidrule(lr){8-11} \cmidrule(lr){12-14} \cmidrule(lr){15-18}
\textbf{Pool} & \textbf{Src} & \textbf{\#Img} & \rotatebox{60}{AI2D} & \rotatebox{60}{SciQA} & \rotatebox{60}{MathV} & \rotatebox{60}{MMMU} & \rotatebox{60}{GQA} & \rotatebox{60}{MME} & \rotatebox{60}{SEED$^\text{L}$} & \rotatebox{60}{MMB} & \rotatebox{60}{OCR-Bv2} & \rotatebox{60}{DocVQA} & \rotatebox{60}{ChartQA} & \rotatebox{60}{CV-2D} & \rotatebox{60}{CV-3D} & \rotatebox{60}{BLINK} & \rotatebox{60}{RealWQ} & \cellcolor{avgrow}\textbf{Avg} & \textbf{\%BL} \\
\midrule
\cellcolor{baserow}Baseline$^\dag$ & \cellcolor{baserow}LCS & \cellcolor{baserow}558k & \cellcolor{baserow}46.5 & \cellcolor{baserow}66.2 & \cellcolor{baserow}26.9 & \cellcolor{baserow}33.6 & \cellcolor{baserow}59.7 & \cellcolor{baserow}58.4 & \cellcolor{baserow}56.2 & \cellcolor{baserow}46.1 & \cellcolor{baserow}10.1 & \cellcolor{baserow}12.8 & \cellcolor{baserow}10.8 & \cellcolor{baserow}53.5 & \cellcolor{baserow}56.4 & \cellcolor{baserow}37.2 & \cellcolor{baserow}47.3 & \cellcolor{baserow}41.4 & \cellcolor{baserow}100.0 \\
Random & -- & 0 & 46.6 & 66.9 & 26.3 & 34.2 & 59.8 & 58.1 & 56.0 & 45.2 & 10.4 & 12.8 & 11.0 & 54.2 & 54.9 & 37.4 & 46.9 & \cellcolor{avgrow}41.4 & 99.8 \\
\midrule
PORTAL & Cambrian & 4096 & \textbf{48.2} & 66.7 & 26.9 & \textbf{34.3} & 62.0 & 64.3 & \textbf{59.6} & 53.6 & 10.8 & \textbf{16.8} & \textbf{12.0} & 53.5 & 55.8 & 37.7 & 49.4 & \cellcolor{avgrow}43.4 & 104.8 \\
PORTAL & COCO & 1024 & 47.6 & \textbf{67.1} & \textbf{29.2} & 33.7 & \textbf{62.1} & \textbf{66.0} & \textbf{59.6} & \textbf{54.2} & \textbf{11.4} & 16.3 & \textbf{12.0} & \textbf{54.6} & \textbf{57.3} & \textbf{38.9} & \textbf{49.7} & \cellcolor{avgrow}\textbf{44.0} & \textbf{106.1} \\
\bottomrule
\end{tabular}}
\end{table}

\section{Projector Architecture: Linear vs.\ MLP}
\label{app:linear_vs_mlp}

The standard LLaVA projector is a two-layer MLP with a GELU activation. We test whether the benefit of PORTAL also transfers to a simpler architecture. We replace the MLP with a single linear layer (no nonlinearity, no hidden expansion, input dimension $1024$, output dimension $d_{\text{LLM}}$) and keep all other components fixed: LLaMA~3.2~1B/CLIP, Stage-II tuning on LLaVA-665k, identical hyperparameters, and the same PORTAL construction with rank $k{=}512$ on the COCO-1024 image pool from Appendix~\ref{app:image_pool_ablation}. For each projector we evaluate three initialization schemes: a trained baseline (Stage-I on LCS-558k), Random init, and PORTAL. The MLP rows are reproduced from Table~\ref{tab:main_results}.

Table~\ref{tab:linear_vs_mlp} reports the per-benchmark breakdown. Under the MLP, Random init matches the trained baseline ($41.5$ vs.\ $41.6$ Avg, $99.9\%$) and PORTAL reaches $44.2$ Avg ($106.2\%$). Under the linear projector, Random init falls to $37.6$ Avg ($94.2\%$ of the linear baseline) and PORTAL reaches $44.7$ Avg ($111.8\%$), the highest macro average among the six configurations. The Random-to-PORTAL gap on the macro average is $2.7$ points under the MLP and $7.1$ points under the linear projector. PORTAL is the stronger zero-shot init on $15$ of $16$ benchmarks under the MLP and on $16$ of $16$ benchmarks under the linear projector. This supports the interpretation that the gain from PORTAL is a property of the projector initialization and transfers across projector architectures.

\begin{table}[!htbp]
\centering
\caption{\textbf{Linear versus MLP-2x-GELU projector, per-benchmark breakdown.} Setup: LLaMA~3.2~1B/CLIP, Stage-II on LLaVA-665k. Each architecture is anchored to its own trained baseline for \%BL. \textbf{Bold}: best between Random and PORTAL within an architecture. \underline{Underline}: second best.}
\label{tab:linear_vs_mlp}
\vspace{2pt}
\resizebox{\textwidth}{!}{%
\setlength{\tabcolsep}{3.5pt}
\renewcommand{\arraystretch}{1.08}
\begin{tabular}{@{}l l G cccc G cccc G cccc G cccc G c c@{}}
\toprule
& & \multicolumn{4}{c}{Knowledge} & \multicolumn{4}{c}{General} & \multicolumn{4}{c}{OCR} & \multicolumn{4}{c}{Vision} & & \\
\cmidrule(lr){3-6} \cmidrule(lr){7-10} \cmidrule(lr){11-14} \cmidrule(lr){15-18}
\textbf{Projector} & \textbf{Init} & \rotatebox{60}{AI2D} & \rotatebox{60}{SciQA} & \rotatebox{60}{MathV} & \rotatebox{60}{MMMU} & \rotatebox{60}{GQA} & \rotatebox{60}{MME} & \rotatebox{60}{SEED$^\text{L}$} & \rotatebox{60}{MMB} & \rotatebox{60}{TxtVQA} & \rotatebox{60}{OCR-Bv2} & \rotatebox{60}{DocVQA} & \rotatebox{60}{ChartQA} & \rotatebox{60}{CV-2D} & \rotatebox{60}{CV-3D} & \rotatebox{60}{BLINK} & \rotatebox{60}{RealWQ} & \cellcolor{avgrow}\textbf{Avg} & \textbf{\%BL} \\
\midrule
\multirow{3}{*}{MLP-2x-GELU}
 & \cellcolor{baserow}Baseline$^\dag$ & \cellcolor{baserow}46.5 & \cellcolor{baserow}66.2 & \cellcolor{baserow}26.9 & \cellcolor{baserow}33.6 & \cellcolor{baserow}59.7 & \cellcolor{baserow}58.4 & \cellcolor{baserow}56.2 & \cellcolor{baserow}46.1 & \cellcolor{baserow}43.7 & \cellcolor{baserow}10.1 & \cellcolor{baserow}12.8 & \cellcolor{baserow}10.8 & \cellcolor{baserow}53.5 & \cellcolor{baserow}56.4 & \cellcolor{baserow}37.2 & \cellcolor{baserow}47.3 & \cellcolor{baserow}41.6 & \cellcolor{baserow}100.0 \\
 & Random       & \underline{46.6} & \underline{66.9} & \underline{26.3} & \cellcolor{bestcell}\textbf{34.2} & \underline{59.8} & \underline{58.1} & \underline{56.0} & \underline{45.2} & \underline{43.8} & \underline{10.4} & \underline{12.8} & \underline{11.0} & \underline{54.2} & \underline{54.9} & \underline{37.4} & \underline{46.9} & \cellcolor{avgrow}\underline{41.5} & \underline{99.9} \\
 & PORTAL & \cellcolor{bestcell}\textbf{47.6} & \cellcolor{bestcell}\textbf{67.1} & \cellcolor{bestcell}\textbf{29.2} & \underline{33.7} & \cellcolor{bestcell}\textbf{62.1} & \cellcolor{bestcell}\textbf{66.0} & \cellcolor{bestcell}\textbf{59.6} & \cellcolor{bestcell}\textbf{54.2} & \cellcolor{bestcell}\textbf{47.1} & \cellcolor{bestcell}\textbf{11.4} & \cellcolor{bestcell}\textbf{16.3} & \cellcolor{bestcell}\textbf{12.0} & \cellcolor{bestcell}\textbf{54.6} & \cellcolor{bestcell}\textbf{57.3} & \cellcolor{bestcell}\textbf{38.9} & \cellcolor{bestcell}\textbf{49.7} & \cellcolor{avgrow}\cellcolor{bestcell}\textbf{44.2} & \cellcolor{bestcell}\textbf{106.2} \\
\midrule
\multirow{3}{*}{Linear}
 & \cellcolor{baserow}Baseline$^\dag$ & \cellcolor{baserow}46.5 & \cellcolor{baserow}65.5 & \cellcolor{baserow}27.1 & \cellcolor{baserow}33.3 & \cellcolor{baserow}54.7 & \cellcolor{baserow}53.1 & \cellcolor{baserow}47.3 & \cellcolor{baserow}49.8 & \cellcolor{baserow}42.6 & \cellcolor{baserow}10.1 & \cellcolor{baserow}11.0 & \cellcolor{baserow}10.0 & \cellcolor{baserow}50.1 & \cellcolor{baserow}54.4 & \cellcolor{baserow}36.8 & \cellcolor{baserow}46.7 & \cellcolor{baserow}39.9 & \cellcolor{baserow}100.0 \\
 & Random       & \underline{46.0} & \underline{63.2} & \underline{24.6} & \underline{30.9} & \underline{53.0} & \underline{52.8} & \underline{46.1} & \underline{43.8} & \underline{36.8} & \underline{7.9} & \underline{7.3} & \underline{8.8} & \underline{45.6} & \underline{52.8} & \underline{37.7} & \underline{44.4} & \cellcolor{avgrow}\underline{37.6} & \underline{94.2} \\
 & PORTAL & \cellcolor{bestcell}\textbf{48.2} & \cellcolor{bestcell}\textbf{68.5} & \cellcolor{bestcell}\textbf{28.2} & \cellcolor{bestcell}\textbf{34.6} & \cellcolor{bestcell}\textbf{62.2} & \cellcolor{bestcell}\textbf{66.8} & \cellcolor{bestcell}\textbf{60.4} & \cellcolor{bestcell}\textbf{54.4} & \cellcolor{bestcell}\textbf{48.9} & \cellcolor{bestcell}\textbf{11.6} & \cellcolor{bestcell}\textbf{17.5} & \cellcolor{bestcell}\textbf{12.1} & \cellcolor{bestcell}\textbf{54.1} & \cellcolor{bestcell}\textbf{57.8} & \cellcolor{bestcell}\textbf{38.1} & \cellcolor{bestcell}\textbf{51.1} & \cellcolor{avgrow}\cellcolor{bestcell}\textbf{44.7} & \textbf{111.8} \\
\bottomrule
\end{tabular}}
\end{table}

\section{PCA Rank Sensitivity}
\label{app:rank_ablation}

\begin{figure}[!htbp]
\centering
\includegraphics[width=0.6\linewidth]{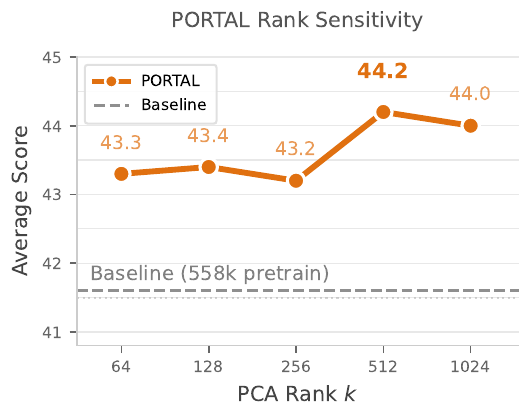}
\caption{\textbf{PCA rank sensitivity} on LLaMA~3.2~1B/CLIP. Average over the 16 benchmarks as a function of the rank $k$ used in the PORTAL construction, with identical Stage-II tuning. All ranks outperform the trained baseline ($41.6$). $k{=}512$ is optimal ($44.2$).}
\label{fig:rank_ablation}
\end{figure}

The PORTAL projector is constructed in a rank-$k$ PCA subspace (Section~\ref{sec:method_construction}). We vary $k$ over $\{64, 128, 256, 512, 1024\}$ on LLaMA~3.2~1B/CLIP, constructing each projector from the same $1024$ COCO images and running identical Stage-II tuning. Figure~\ref{fig:rank_ablation} shows the resulting Avg score. Performance is flat from $k{=}64$ to $k{=}256$ (approximately $43.3$ Avg), peaks at $k{=}512$ ($44.2$), and decreases marginally at $k{=}1024$ ($44.0$). The plateau at low ranks indicates that the dominant PCA directions already capture most of the cross-modal structure. The peak at $k{=}512$ suggests that mid-variance directions contribute additional geometric signal. The decrease at $k{=}1024$ is consistent with noisy eigenvectors diluting the transport map. We adopt $k{=}512$ as the default throughout the paper.

\newpage
\section*{NeurIPS Paper Checklist}

\begin{enumerate}

\item {\bf Claims}
    \item[] Question: Do the main claims made in the abstract and introduction accurately reflect the paper's contributions and scope?
    \item[] Answer: \answerYes{}
    \item[] Justification: Yes, all claims are placed against related work (Section~\ref{sec:RelatedWork}) and all claims regarding the method's performance are backed by empirical results (Section~\ref{sec:Experiments}).
    \item[] Guidelines:
    \begin{itemize}
        \item The answer \answerNA{} means that the abstract and introduction do not include the claims made in the paper.
        \item The abstract and/or introduction should clearly state the claims made, including the contributions made in the paper and important assumptions and limitations. A \answerNo{} or \answerNA{} answer to this question will not be perceived well by the reviewers. 
        \item The claims made should match theoretical and experimental results, and reflect how much the results can be expected to generalize to other settings. 
        \item It is fine to include aspirational goals as motivation as long as it is clear that these goals are not attained by the paper. 
    \end{itemize}

\item {\bf Limitations}
    \item[] Question: Does the paper discuss the limitations of the work performed by the authors?
    \item[] Answer: \answerYes{}
    \item[] Justification: Yes, the paper discusses limitations in the Conclusion and in Appendix~\ref{sec:Limitations}.
    \item[] Guidelines:
    \begin{itemize}
        \item The answer \answerNA{} means that the paper has no limitation while the answer \answerNo{} means that the paper has limitations, but those are not discussed in the paper. 
        \item The authors are encouraged to create a separate ``Limitations'' section in their paper.
        \item The paper should point out any strong assumptions and how robust the results are to violations of these assumptions (e.g., independence assumptions, noiseless settings, model well-specification, asymptotic approximations only holding locally). The authors should reflect on how these assumptions might be violated in practice and what the implications would be.
        \item The authors should reflect on the scope of the claims made, e.g., if the approach was only tested on a few datasets or with a few runs. In general, empirical results often depend on implicit assumptions, which should be articulated.
        \item The authors should reflect on the factors that influence the performance of the approach. For example, a facial recognition algorithm may perform poorly when image resolution is low or images are taken in low lighting. Or a speech-to-text system might not be used reliably to provide closed captions for online lectures because it fails to handle technical jargon.
        \item The authors should discuss the computational efficiency of the proposed algorithms and how they scale with dataset size.
        \item If applicable, the authors should discuss possible limitations of their approach to address problems of privacy and fairness.
        \item While the authors might fear that complete honesty about limitations might be used by reviewers as grounds for rejection, a worse outcome might be that reviewers discover limitations that aren't acknowledged in the paper. The authors should use their best judgment and recognize that individual actions in favor of transparency play an important role in developing norms that preserve the integrity of the community. Reviewers will be specifically instructed to not penalize honesty concerning limitations.
    \end{itemize}

\item {\bf Theory assumptions and proofs}
    \item[] Question: For each theoretical result, does the paper provide the full set of assumptions and a complete (and correct) proof?
    \item[] Answer: \answerYes{}
    \item[] Justification: Yes, Proposition~\ref{prop:sigma_c} states its assumptions, and Appendix~\ref{app:concentration} gives the complete proof.
    \item[] Guidelines:
    \begin{itemize}
        \item The answer \answerNA{} means that the paper does not include theoretical results. 
        \item All the theorems, formulas, and proofs in the paper should be numbered and cross-referenced.
        \item All assumptions should be clearly stated or referenced in the statement of any theorems.
        \item The proofs can either appear in the main paper or the supplemental material, but if they appear in the supplemental material, the authors are encouraged to provide a short proof sketch to provide intuition. 
        \item Inversely, any informal proof provided in the core of the paper should be complemented by formal proofs provided in appendix or supplemental material.
        \item Theorems and Lemmas that the proof relies upon should be properly referenced. 
    \end{itemize}

    \item {\bf Experimental result reproducibility}
    \item[] Question: Does the paper fully disclose all the information needed to reproduce the main experimental results of the paper to the extent that it affects the main claims and/or conclusions of the paper (regardless of whether the code and data are provided or not)?
    \item[] Answer: \answerYes{}
    \item[] Justification: Yes, experimental details are described in the experimental setup (Section~\ref{sec:Experiments}), and all hyperparameters and the evaluation protocol are given in Appendices~\ref{app:compute} and~\ref{app:eval_setup}.
    \item[] Guidelines:
    \begin{itemize}
        \item The answer \answerNA{} means that the paper does not include experiments.
        \item If the paper includes experiments, a \answerNo{} answer to this question will not be perceived well by the reviewers: Making the paper reproducible is important, regardless of whether the code and data are provided or not.
        \item If the contribution is a dataset and\slash or model, the authors should describe the steps taken to make their results reproducible or verifiable. 
        \item Depending on the contribution, reproducibility can be accomplished in various ways. For example, if the contribution is a novel architecture, describing the architecture fully might suffice, or if the contribution is a specific model and empirical evaluation, it may be necessary to either make it possible for others to replicate the model with the same dataset, or provide access to the model. In general. releasing code and data is often one good way to accomplish this, but reproducibility can also be provided via detailed instructions for how to replicate the results, access to a hosted model (e.g., in the case of a large language model), releasing of a model checkpoint, or other means that are appropriate to the research performed.
        \item While NeurIPS does not require releasing code, the conference does require all submissions to provide some reasonable avenue for reproducibility, which may depend on the nature of the contribution. For example
        \begin{enumerate}
            \item If the contribution is primarily a new algorithm, the paper should make it clear how to reproduce that algorithm.
            \item If the contribution is primarily a new model architecture, the paper should describe the architecture clearly and fully.
            \item If the contribution is a new model (e.g., a large language model), then there should either be a way to access this model for reproducing the results or a way to reproduce the model (e.g., with an open-source dataset or instructions for how to construct the dataset).
            \item We recognize that reproducibility may be tricky in some cases, in which case authors are welcome to describe the particular way they provide for reproducibility. In the case of closed-source models, it may be that access to the model is limited in some way (e.g., to registered users), but it should be possible for other researchers to have some path to reproducing or verifying the results.
        \end{enumerate}
    \end{itemize}

\item {\bf Open access to data and code}
    \item[] Question: Does the paper provide open access to the data and code, with sufficient instructions to faithfully reproduce the main experimental results, as described in supplemental material?
    \item[] Answer: \answerYes{}
    \item[] Justification: Yes, the code is available at \href{https://github.com/mlbio-epfl/PORTAL}{https://github.com/mlbio-epfl/PORTAL}.
    \item[] Guidelines:
    \begin{itemize}
        \item The answer \answerNA{} means that paper does not include experiments requiring code.
        \item Please see the NeurIPS code and data submission guidelines (\url{https://neurips.cc/public/guides/CodeSubmissionPolicy}) for more details.
        \item While we encourage the release of code and data, we understand that this might not be possible, so \answerNo{} is an acceptable answer. Papers cannot be rejected simply for not including code, unless this is central to the contribution (e.g., for a new open-source benchmark).
        \item The instructions should contain the exact command and environment needed to run to reproduce the results. See the NeurIPS code and data submission guidelines (\url{https://neurips.cc/public/guides/CodeSubmissionPolicy}) for more details.
        \item The authors should provide instructions on data access and preparation, including how to access the raw data, preprocessed data, intermediate data, and generated data, etc.
        \item The authors should provide scripts to reproduce all experimental results for the new proposed method and baselines. If only a subset of experiments are reproducible, they should state which ones are omitted from the script and why.
        \item At submission time, to preserve anonymity, the authors should release anonymized versions (if applicable).
        \item Providing as much information as possible in supplemental material (appended to the paper) is recommended, but including URLs to data and code is permitted.
    \end{itemize}

\item {\bf Experimental setting/details}
    \item[] Question: Does the paper specify all the training and test details (e.g., data splits, hyperparameters, how they were chosen, type of optimizer) necessary to understand the results?
    \item[] Answer: \answerYes{}
    \item[] Justification: Yes, experimental details are described in the experimental setup (Section~\ref{sec:Experiments}), and all hyperparameters and the evaluation protocol are given in Appendices~\ref{app:compute} and~\ref{app:eval_setup}.
    \item[] Guidelines:
    \begin{itemize}
        \item The answer \answerNA{} means that the paper does not include experiments.
        \item The experimental setting should be presented in the core of the paper to a level of detail that is necessary to appreciate the results and make sense of them.
        \item The full details can be provided either with the code, in appendix, or as supplemental material.
    \end{itemize}

\item {\bf Experiment statistical significance}
    \item[] Question: Does the paper report error bars suitably and correctly defined or other appropriate information about the statistical significance of the experiments?
    \item[] Answer: \answerYes{}
    \item[] Justification: Yes, all aggregate scores in the main results and ablation tables are reported with $95\%$ confidence intervals computed from $B{=}1{,}000$ paired bootstrap resamples~\cite{efron1979bootstrap}; resampling is paired across methods within each (encoder, LLM) group so the intervals reflect method differences rather than sample noise. The $\pm$ subscripts in the tables denote the half-width of these CIs. The bootstrap protocol is described in Appendix~\ref{app:eval_setup}.
    \item[] Guidelines:
    \begin{itemize}
        \item The answer \answerNA{} means that the paper does not include experiments.
        \item The authors should answer \answerYes{} if the results are accompanied by error bars, confidence intervals, or statistical significance tests, at least for the experiments that support the main claims of the paper.
        \item The factors of variability that the error bars are capturing should be clearly stated (for example, train/test split, initialization, random drawing of some parameter, or overall run with given experimental conditions).
        \item The method for calculating the error bars should be explained (closed form formula, call to a library function, bootstrap, etc.)
        \item The assumptions made should be given (e.g., Normally distributed errors).
        \item It should be clear whether the error bar is the standard deviation or the standard error of the mean.
        \item It is OK to report 1-sigma error bars, but one should state it. The authors should preferably report a 2-sigma error bar than state that they have a 96\% CI, if the hypothesis of Normality of errors is not verified.
        \item For asymmetric distributions, the authors should be careful not to show in tables or figures symmetric error bars that would yield results that are out of range (e.g., negative error rates).
        \item If error bars are reported in tables or plots, the authors should explain in the text how they were calculated and reference the corresponding figures or tables in the text.
    \end{itemize}

\item {\bf Experiments compute resources}
    \item[] Question: For each experiment, does the paper provide sufficient information on the computer resources (type of compute workers, memory, time of execution) needed to reproduce the experiments?
    \item[] Answer: \answerYes{}
    \item[] Justification: Yes, we specify the computational resources in Appendix~\ref{app:compute}.
    \item[] Guidelines:
    \begin{itemize}
        \item The answer \answerNA{} means that the paper does not include experiments.
        \item The paper should indicate the type of compute workers CPU or GPU, internal cluster, or cloud provider, including relevant memory and storage.
        \item The paper should provide the amount of compute required for each of the individual experimental runs as well as estimate the total compute. 
        \item The paper should disclose whether the full research project required more compute than the experiments reported in the paper (e.g., preliminary or failed experiments that didn't make it into the paper). 
    \end{itemize}
    
\item {\bf Code of ethics}
    \item[] Question: Does the research conducted in the paper conform, in every respect, with the NeurIPS Code of Ethics \url{https://neurips.cc/public/EthicsGuidelines}?
    \item[] Answer: \answerYes{}
    \item[] Justification: Yes, we have carefully reviewed and adhered to the NeurIPS Code of Ethics.
    \item[] Guidelines:
    \begin{itemize}
        \item The answer \answerNA{} means that the authors have not reviewed the NeurIPS Code of Ethics.
        \item If the authors answer \answerNo, they should explain the special circumstances that require a deviation from the Code of Ethics.
        \item The authors should make sure to preserve anonymity (e.g., if there is a special consideration due to laws or regulations in their jurisdiction).
    \end{itemize}

\item {\bf Broader impacts}
    \item[] Question: Does the paper discuss both potential positive societal impacts and negative societal impacts of the work performed?
    \item[] Answer: \answerYes{}
    \item[] Justification:  Yes, the positive societal impact of the paper is discussed in the introduction and conclusion. Negative societal impacts are not expected as the outlined methodology enables multimodal LLMs projector alignment without paired data while being computationally efficient.
    \item[] Guidelines:
    \begin{itemize}
        \item The answer \answerNA{} means that there is no societal impact of the work performed.
        \item If the authors answer \answerNA{} or \answerNo, they should explain why their work has no societal impact or why the paper does not address societal impact.
        \item Examples of negative societal impacts include potential malicious or unintended uses (e.g., disinformation, generating fake profiles, surveillance), fairness considerations (e.g., deployment of technologies that could make decisions that unfairly impact specific groups), privacy considerations, and security considerations.
        \item The conference expects that many papers will be foundational research and not tied to particular applications, let alone deployments. However, if there is a direct path to any negative applications, the authors should point it out. For example, it is legitimate to point out that an improvement in the quality of generative models could be used to generate Deepfakes for disinformation. On the other hand, it is not needed to point out that a generic algorithm for optimizing neural networks could enable people to train models that generate Deepfakes faster.
        \item The authors should consider possible harms that could arise when the technology is being used as intended and functioning correctly, harms that could arise when the technology is being used as intended but gives incorrect results, and harms following from (intentional or unintentional) misuse of the technology.
        \item If there are negative societal impacts, the authors could also discuss possible mitigation strategies (e.g., gated release of models, providing defenses in addition to attacks, mechanisms for monitoring misuse, mechanisms to monitor how a system learns from feedback over time, improving the efficiency and accessibility of ML).
    \end{itemize}
    
\item {\bf Safeguards}
    \item[] Question: Does the paper describe safeguards that have been put in place for responsible release of data or models that have a high risk for misuse (e.g., pre-trained language models, image generators, or scraped datasets)?
    \item[] Answer: \answerNA{}
    \item[] Justification:  This paper poses no such risks as beyond code, we are not releasing any new artifacts
such as data or models.
    \item[] Guidelines:
    \begin{itemize}
        \item The answer \answerNA{} means that the paper poses no such risks.
        \item Released models that have a high risk for misuse or dual-use should be released with necessary safeguards to allow for controlled use of the model, for example by requiring that users adhere to usage guidelines or restrictions to access the model or implementing safety filters. 
        \item Datasets that have been scraped from the Internet could pose safety risks. The authors should describe how they avoided releasing unsafe images.
        \item We recognize that providing effective safeguards is challenging, and many papers do not require this, but we encourage authors to take this into account and make a best faith effort.
    \end{itemize}

\item {\bf Licenses for existing assets}
    \item[] Question: Are the creators or original owners of assets (e.g., code, data, models), used in the paper, properly credited and are the license and terms of use explicitly mentioned and properly respected?
    \item[] Answer: \answerYes{}
    \item[] Justification: Yes, we cite the creators of every model and dataset we use (Section~\ref{sec:Experiments}, Appendices~\ref{app:compute} and~\ref{app:eval_setup}), and Table~\ref{tab:datasets} lists the license of each dataset.
    \item[] Guidelines:
    \begin{itemize}
        \item The answer \answerNA{} means that the paper does not use existing assets.
        \item The authors should cite the original paper that produced the code package or dataset.
        \item The authors should state which version of the asset is used and, if possible, include a URL.
        \item The name of the license (e.g., CC-BY 4.0) should be included for each asset.
        \item For scraped data from a particular source (e.g., website), the copyright and terms of service of that source should be provided.
        \item If assets are released, the license, copyright information, and terms of use in the package should be provided. For popular datasets, \url{paperswithcode.com/datasets} has curated licenses for some datasets. Their licensing guide can help determine the license of a dataset.
        \item For existing datasets that are re-packaged, both the original license and the license of the derived asset (if it has changed) should be provided.
        \item If this information is not available online, the authors are encouraged to reach out to the asset's creators.
    \end{itemize}

\item {\bf New assets}
    \item[] Question: Are new assets introduced in the paper well documented and is the documentation provided alongside the assets?
    \item[] Answer: \answerNA{}
    \item[] Justification:  This paper does not release any new assets beyond code, which is adequately documented in order to reproduce all results from the paper.
    \item[] Guidelines:
    \begin{itemize}
        \item The answer \answerNA{} means that the paper does not release new assets.
        \item Researchers should communicate the details of the dataset\slash code\slash model as part of their submissions via structured templates. This includes details about training, license, limitations, etc. 
        \item The paper should discuss whether and how consent was obtained from people whose asset is used.
        \item At submission time, remember to anonymize your assets (if applicable). You can either create an anonymized URL or include an anonymized zip file.
    \end{itemize}

\item {\bf Crowdsourcing and research with human subjects}
    \item[] Question: For crowdsourcing experiments and research with human subjects, does the paper include the full text of instructions given to participants and screenshots, if applicable, as well as details about compensation (if any)? 
    \item[] Answer: \answerNA{}
    \item[] Justification:  The paper does not involve crowdsourcing nor research with human subjects.
    \item[] Guidelines:
    \begin{itemize}
        \item The answer \answerNA{} means that the paper does not involve crowdsourcing nor research with human subjects.
        \item Including this information in the supplemental material is fine, but if the main contribution of the paper involves human subjects, then as much detail as possible should be included in the main paper. 
        \item According to the NeurIPS Code of Ethics, workers involved in data collection, curation, or other labor should be paid at least the minimum wage in the country of the data collector. 
    \end{itemize}

\item {\bf Institutional review board (IRB) approvals or equivalent for research with human subjects}
    \item[] Question: Does the paper describe potential risks incurred by study participants, whether such risks were disclosed to the subjects, and whether Institutional Review Board (IRB) approvals (or an equivalent approval/review based on the requirements of your country or institution) were obtained?
    \item[] Answer: \answerNA{}
    \item[] Justification: The paper does not involve crowdsourcing nor research with human subjects.
    \item[] Guidelines:
    \begin{itemize}
        \item The answer \answerNA{} means that the paper does not involve crowdsourcing nor research with human subjects.
        \item Depending on the country in which research is conducted, IRB approval (or equivalent) may be required for any human subjects research. If you obtained IRB approval, you should clearly state this in the paper. 
        \item We recognize that the procedures for this may vary significantly between institutions and locations, and we expect authors to adhere to the NeurIPS Code of Ethics and the guidelines for their institution. 
        \item For initial submissions, do not include any information that would break anonymity (if applicable), such as the institution conducting the review.
    \end{itemize}

\item {\bf Declaration of LLM usage}
    \item[] Question: Does the paper describe the usage of LLMs if it is an important, original, or non-standard component of the core methods in this research? Note that if the LLM is used only for writing, editing, or formatting purposes and does \emph{not} impact the core methodology, scientific rigor, or originality of the research, declaration is not required.
    \item[] Answer: \answerNA{}
    \item[] Justification:  The paper does not involve the development of any LLMs, nor their use beyond minor editing.
    \item[] Guidelines:
    \begin{itemize}
        \item The answer \answerNA{} means that the core method development in this research does not involve LLMs as any important, original, or non-standard components.
        \item Please refer to our LLM policy in the NeurIPS handbook for what should or should not be described.
    \end{itemize}

\end{enumerate}

\end{document}